\documentclass{article}

 \usepackage[main, final]{neurips_2026}

\usepackage[utf8]{inputenc} 
\usepackage[T1]{fontenc}    
\usepackage{hyperref}       
\usepackage{url}            
\usepackage{booktabs}       
\usepackage{amsfonts}       
\usepackage{nicefrac}       
\usepackage{microtype}      
\usepackage[dvipsnames]{xcolor}     
\usepackage{amsmath,amssymb}
\usepackage{amsthm}
\usepackage{algorithm}
\usepackage{algpseudocode}
\usepackage{bm}
\usepackage{graphicx}
\usepackage{booktabs}
\usepackage{caption}
\usepackage{multirow}
\usepackage{wrapfig}
\usepackage[table]{xcolor}
\usepackage{enumitem}

\newtheorem{theorem}{Theorem}[section]
\newtheorem{proposition}[theorem]{Proposition}
\newtheorem{lemma}[theorem]{Lemma}
\newtheorem{corollary}[theorem]{Corollary}
\newtheorem{definition}[theorem]{Definition}

\newtheorem{remark}[theorem]{Remark}
\newcommand{\beststd}[2]{\textbf{#1}\kern0.06em\raisebox{-0.45ex}{\fontsize{6}{6}\selectfont \textpm#2}}
\newcommand{\secondstd}[2]{\underline{#1}\kern0.06em\raisebox{-0.45ex}{\fontsize{6}{6}\selectfont \textpm#2}}
\newcommand{\std}[2]{#1\kern0.06em\raisebox{-0.45ex}{\fontsize{6}{6}\selectfont \textpm#2}}

\usepackage{titlesec}

\titlespacing*{\section}{0pt}{1.4ex}{0.6ex}
\titlespacing*{\subsection}{0pt}{1.1ex}{0.5ex}

\usepackage{placeins}
\usepackage{caption}
\usepackage{comment}

\title{Beyond Empirical Support: Structured Outlier Generation via Sinkhorn Optimal Transport}

\author{
  Haixiang Sun, Andrew L. Liu \\
  Edwardson School of Industrial Engineering\\
  Purdue University\\
  West Lafayette, IN, 47906 \\
  \texttt{\{sun1321, andrewliu\}@purdue.edu} \\
}

\begin{document}

\maketitle

\begin{abstract}
Outliers are essential for evaluating and improving the robustness of machine learning systems, especially when future distributions may differ significantly from historical training data. In high-stakes applications, robustness often depends on rare cases that finite datasets fail to capture, making simple resampling or perturbation insufficient for stress scenario generation. Existing outlier synthesis methods typically rely on sparse neighborhoods, low support latent regions, or classifier boundary crossings, which can be heuristic, unstable, and tied to specific modalities or architectures. We therefore propose Sinkhorn Boundary Outlier Generation (SBOG), a structured framework for latent-space outlier generation that couples Sinkhorn optimal transport geometry with distributionally robust boundary modeling. The resulting Sinkhorn-induced support cost guides the sampler toward weakly supported boundary regions, while semantic constraints prevent uncontrolled drift from the intended context, yielding controlled deviations from the in-distribution reference measure rather than arbitrary sparse-region samples. Experiments on time series anomaly generation and image outlier synthesis show that our framework produces informative, semantically controlled outliers and improves downstream robustness evaluation across modalities, providing a foundation for stress scenario generation beyond empirical support.
\end{abstract}



\section{Introduction}
Outliers play an important role in building reliable machine learning systems. In many real-world applications, rare or atypical cases are often precisely the ones that determine whether a model behaves robustly under deployment. Yet such outliers are inherently difficult to study since they occur infrequently, can vary substantially in form, and therefore are poorly covered by finite training datasets. Consequently, most existing work treats outliers primarily as anomalies to be detected, monitored, and mitigated \cite{pmlr-v80-ruff18a,liu2020energy}. Although this line of work is indispensable, it leaves open a complementary question of equal importance: how to model and generate outliers themselves. Addressing this question is important for understanding the structure and mechanisms underlying rare events. It also enables systematic construction of informative training and evaluation cases beyond standard data collection, which can strengthen model robustness, reveal failure modes before deployment, and support safer decision-making under uncertainty.

However, dedicated work on outlier generation remains relatively scarce. Existing approaches typically first learn a representation space and then identify candidate boundary samples using empirical support criteria, such as sparse nearest-neighbor neighborhoods~\cite{npos}, low-likelihood latent regions~\cite{du2023dream}, classifier-boundary crossing rules~\cite{liao2025bood}, or LLM-based heuristics \cite{abbas2025outofdistribution}. These candidates are then perturbed, sampled, or decoded back to the input space for synthetic outliers. Although effective in practice, such criteria often treat low support as the primary signal of outlyingness. Low support alone, however, is not sufficient for useful outlier generation. A point may be sparse because it lies near an informative class boundary, but it may also be sparse because it drifts off the semantic manifold, moves toward another class, or becomes difficult for the generator to decode into a meaningful sample. Therefore, for outlier generation, the relevant objective is not merely to find low-density feature points, but to find weakly supported candidates that remain semantically anchored and generatively valid. Moreover, existing synthesis pipelines are often tailored to a specific modality or generator, with many methods designed primarily for time-series anomaly generation~\cite{darban2025carla,app14177714,darban2025genias} or image outlier synthesis~\cite{du2023dream,gao2025good,xu2024calibrated}. As a result, these modality-specific designs may not generalize well when the underlying data distribution or modality changes.

Motivated by this perspective, we use Sinkhorn optimal transport \cite{cuturi2013sinkhorn} as a support-aware geometry for structured outlier generation in latent space. Given an in-distribution reference measure and a ground cost induced by the learned representation, the entropic transport geometry provides a smooth way to measure how much in-distribution mass can support a candidate latent point. This support cost guides generation toward weakly supported boundary regions, while semantic constraints ensure that the generated candidates remain aligned with the intended context rather than drifting into arbitrary sparse regions. Since the construction only requires a latent reference measure and a ground cost, the same principle can be instantiated across different underlying distributions and modalities, including time series and images, without relying on modality-specific neighborhood rules.


Based on this principle, we propose Sinkhorn Boundary Outlier Generation (SBOG), a structured boundary-generation framework that combines transport-based support scoring, boundary-anchor selection and semantic feasibility. The key component of SBOG is Sinkhorn Outlier Energy (SOE), which measures the transport cost of supporting a candidate point from the in-distribution reference measure. Notably, under the empirical Euclidean case, SOE recovers a smooth kernel-support form closely related to Gaussian Kernel Density Estimation (KDE). However, here SOE is not to introduce a standalone density estimator. Rather, SBOG uses SOE as a transport-based boundary score: SOE identifies smooth boundary regions for anchor selection and candidate acceptance, while semantic constraints keep the generated outliers controlled and class-consistent. This use differs from purely statistical sparsity or density-rejection rules, which often yield unstable low-support scores.

Beyond this geometric view, the corresponding Sinkhorn Distributionally Robust Optimization (DRO) formulation provides an optimization-consistent interpretation of the boundary decisions made by SBOG \cite{kuhn2025distributionally}. The one-dimensional dual \cite{wang2025sinkhorn} exposes the same Gibbs normalizer that appears in SOE, linking high-energy candidates to deviations that are expensive to support under the transport uncertainty geometry. We use
this connection to justify SBOG as a model-agnostic boundary-generation procedure: SOE selects transport-expensive candidates, while semantic constraints restrict them to meaningful class-consistent deviations. When a downstream loss is available, the same dual can be used to derive task-aware variants. As a result, SBOG provides a unified sampling framework that can be evaluated through both generation-quality diagnostics and downstream robustness performance. Our main contributions are summarized as follows: \vspace{-6pt}
\begin{itemize}[leftmargin=1.2em]
\item We formulate structured outlier generation as constrained latent boundary generation, where candidates must be weakly supported by the in-distribution reference measure while remaining semantically anchored and generatively valid. This reframes outlier synthesis beyond low-density or sparse-neighborhood search.

\item We develop Sinkhorn Boundary Outlier Generation (SBOG), an algorithm that uses Sinkhorn Outlier Energy (SOE) to select boundary anchors, accept weakly supported candidates under the Sinkhorn transport geometry, enforce semantic feasibility, and decode synthetic outliers. While SOE reduces to a kernel-support score in the empirical Euclidean case, SBOG uses it as a transport-based boundary criterion rather than as a standalone density score.

\item We show that the SOE boundary rule appears as Gibbs-normalization term in the Sinkhorn-DRO dual, linking SBOG's thresholding and energy-maximizing candidate selection to boundary decisions. Across time-series and image benchmarks, SBOG generates informative, semantically controlled outliers and improves downstream robustness evaluation.
\end{itemize}
\vspace{-10pt}
\section{Preliminaries}

\subsection{Sinkhorn Distributionally Robust Optimization}
\label{sec:sinkhorn_dro}

We begin by recalling the Sinkhorn Distributionally Robust Optimization framework \cite{kuhn2025distributionally}, which provides the transport geometry used later to define outlyingness in latent space.  Let $\mathcal X$ be a measurable space, $\mu\in\mathcal P(\mathcal X)$ denote a nominal distribution, $\nu\in\mathcal P(\mathcal X)$ be a reference measure for admissible perturbations, and $d:\mathcal X\times\mathcal X\to\mathbb R_+$ be a ground cost. For $\varepsilon>0$, define $K_\nu(x)
:=
\int_{\mathcal X}
\exp\!\left(-\frac{d(x,\tilde x)}{\varepsilon}\right)
d\nu(\tilde x)$,
and assume $0<K_\nu(x)<\infty$ for $\mu$-almost every $x$. The associated Gibbs kernel is
\begin{equation}
dQ^{\nu}_{\varepsilon,x}(\tilde x)
=
\frac{\exp\!\big(-d(x,\tilde x)/\varepsilon\big)}{K_\nu(x)}
d\nu(\tilde x).
\end{equation}

\begin{definition}[Sinkhorn OT]
For any candidate worst-case distribution $P\in\mathcal P(\mathcal X)$, define the normalized Sinkhorn Optimal Transport (OT) discrepancy
\begin{equation}\label{sinkhorn_dist}
W^\nu_\varepsilon(\mu,P)
=
\inf_{\gamma\in\Gamma(\mu,P)}
\left\{
\mathbb E_{(x,\tilde x)\sim\gamma}\big[d(x,\tilde x)\big]
+
\varepsilon\mathrm{KL}(\gamma\|\mu\otimes\nu)
+
\varepsilon\mathbb E_{x\sim\mu}\big[\log K_\nu(x)\big]
\right\},
\end{equation}
where $\Gamma(\mu,P)$ is the set of couplings with first marginal $\mu$ and second marginal $P$, and $\mathrm{KL}(\cdot\Vert\cdot)$ is the Kullback--Leibler divergence. Equivalently,
\[
W^\nu_\varepsilon(\mu,P)
=
\inf_{\gamma\in\Gamma(\mu,P)}
\varepsilon\mathrm{KL}\!\left(\gamma\middle\|\mu\otimes Q^\nu_\varepsilon\right),
\]
where $(\mu\otimes Q^\nu_\varepsilon)(dx,d\tilde x):=\mu(dx)Q^\nu_{\varepsilon,x}(d\tilde x)$. The normalization by $K_\nu(x)$ makes the minimum discrepancy over all candidate marginals equal to zero. 
\end{definition}

\begin{definition}[Sinkhorn DRO]
\label{def:sinkhorn_dro}
Let $f:\mathcal X\to\mathbb R$ be measurable and let $\rho>0$. The Sinkhorn ambiguity set is the normalized entropic OT ball $\mathcal U_{\rho}^{\nu}(\mu)=\{P\in\mathcal P(\mathcal X):W^\nu_\varepsilon(\mu,P)\le\rho\}$,and the corresponding worst-case expected value is
\begin{equation}
V_\rho(f;\mu,\nu)
=
\sup_{P\in\mathcal U_{\rho}^{\nu}(\mu)}
\mathbb E_{\tilde x\sim P}\big[f(\tilde x)\big].
\end{equation}
\end{definition}\vspace{-10pt}
By adapting the strong-duality result of \cite[Theorem~1]{wang2025sinkhorn}, the Sinkhorn DRO can be simplified as:
\begin{lemma}[One-dimensional Sinkhorn-DRO dual]\label{lem:sinkhorndro}
Assume the ambiguity set is nonempty and the exponential moments below are finite. By strong duality, $V_\rho(f;\mu,\nu)$ admits the one-dimensional dual representation $V_\rho(f;\mu,\nu)=\inf_{\lambda>0}V_D(\lambda)$, with
\begin{equation}\label{eq_dual}
V_D(\lambda)=\lambda\rho+\lambda\varepsilon\;\mathbb E_{x\sim\mu}\left[\log\mathbb E_{\tilde x\sim Q^{\nu}_{\varepsilon,x}}\exp\!\left(\frac{f(\tilde x)}{\lambda\varepsilon}\right)\right].
\end{equation}
\end{lemma}

The kernel $Q^{\nu}_{\varepsilon,x}$ concentrates on low-cost perturbations of $x$, while the log-moment generating term in \eqref{eq_dual} exponentially tilts toward large values of $f$, yielding a softmax-type robustification governed by $(\rho,\varepsilon)$ and the dual temperature $\lambda$.

\subsection{Outlier Generation}



Let $\mathcal X$ denote the input space, and let
$D_{\rm in}=\{(x_i,y_i)\}_{i=1}^n$ with $y_i\in\{1,\ldots,C\}$ be an
in-distribution dataset. Let $T:\{1,\ldots,C\}\to\mathbb R^m$ be a fixed conditioning map, with $t_c:=T(c)$. Let $h_\theta:\mathcal X\to\mathbb R^m$ be a trainable encoder and define the normalized representation $z(x;\theta):=\frac{h_\theta(x)}{\|h_\theta(x)\|_2}\in\mathcal Z\subseteq\mathbb R^m$. We write $z_i:=z(x_i;\theta)$ and define the empirical In-Distriubtion (ID) embedding set $Z_n:=\{z_i\}_{i=1}^n\subset\mathcal Z$. The first stage learns $\theta$ so that each $z_i$ is aligned with its corresponding conditioning vector $t_{y_i}$. The empirical latent reference measure is $\hat\nu_n := \frac1n\sum_{i=1}^n \delta_{z_i}$. Outlier generation can therefore be viewed as the problem of generating
new latent points relative to the geometry induced by the embedding set.

Given the learned latent space $\mathcal Z$ and the empirical ID set $Z_n$,
let $\mathcal I(Z_n)\subseteq\mathcal Z$ denote the region regarded as
in-distribution. An outlier sampler is a conditional distribution
$S(\cdot\mid Z_n)$ over $\mathcal Z$ designed to place most of its mass outside $\mathcal I(Z_n)$. It produces latent candidates $v$ at least approximately.  Each sampled latent point $v$ is then mapped back to the input space through a fixed conditional generator $x_{\mathrm{ood}}\sim p(x\mid v)$, where OOD denotes out-of-distribution.. Repeating this procedure yields an auxiliary set $\mathcal{D}_{\mathrm{ood}}=\{x_{\mathrm{ood}}\}$.


\section{SBOG: Sinkhorn Boundary Outlier Generation}
\begin{figure}[t]
    \centering
    \includegraphics[width=1\linewidth]{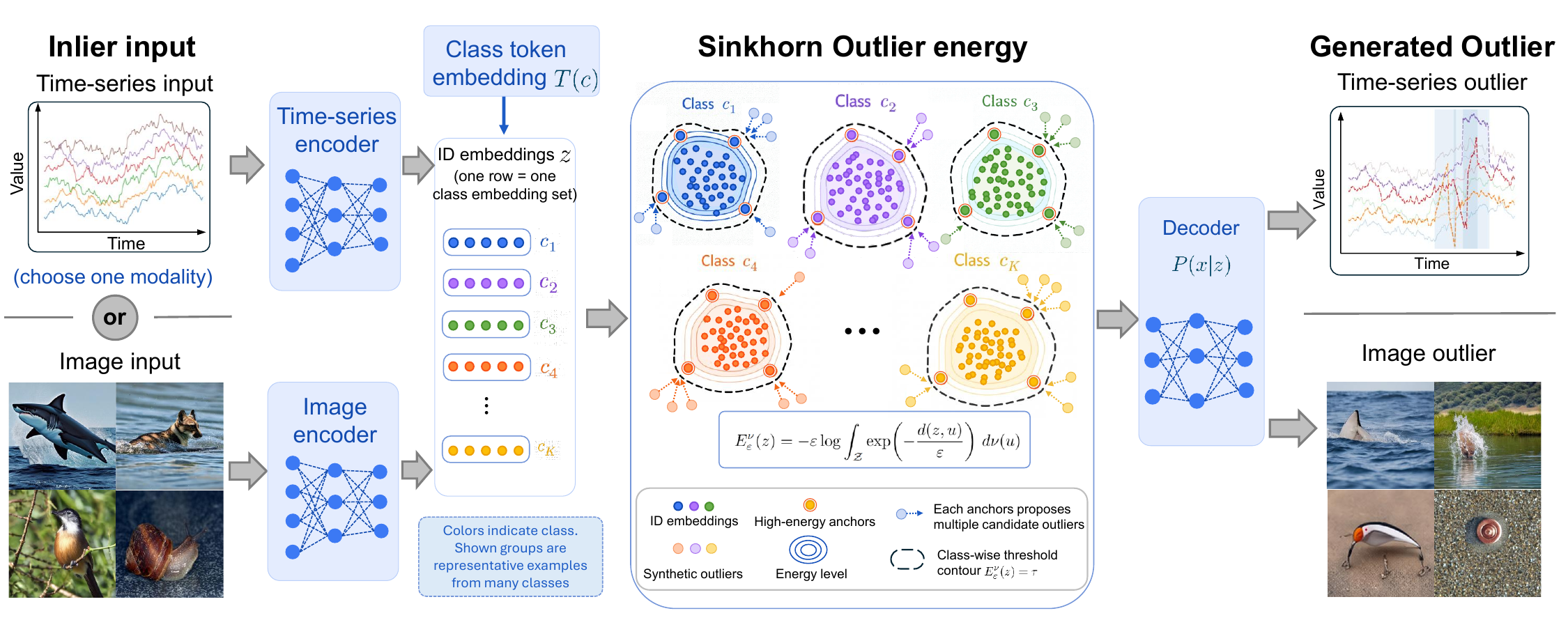}
    \caption{Pipeline of Sinkhorn Boundary Outlier Generation (SBOG).}\vspace{-10pt}
    \label{fig:placeholder}
\end{figure}
\subsection{Outlier Sampling under Latent Space}
\label{sec:outlier_sampling}

We instantiate the sampler $S(\cdot\mid Z_n)$ as Sinkhorn Boundary Outlier Generation (SBOG), a latent-space procedure for generating structured outliers. We work in the latent representation space $\mathcal Z$ and generate synthetic outliers $v\in\mathcal Z$. This is reasonable because alignment-trained encoders tend to organize representations semantically: samples from the same class concentrate around their class prototype, while different classes become separated in latent space \cite{neuralcollapse}. We therefore assume that the learned representation space $\mathcal Z$ is sufficiently aligned so that the ground cost $d:\mathcal Z\times\mathcal Z\to\mathbb R_+$ reflects local semantic geometry. Under this geometry, SBOG does not define outlyingness through a purely statistical proxy such as nearest-neighbor sparsity or classifier-boundary crossing. Instead, it measures how expensive it is to support a candidate latent point using mass from the in-distribution reference measure under an entropic transport geometry. This gives a smooth boundary-support criterion that can be instantiated across different underlying distributions and modalities.

The key support component of SBOG is a transport-based energy, which we introduce first. Let \(\nu \in \mathcal P(\mathcal Z)\) denote the in-distribution reference measure in the latent representation space, for a candidate latent point \(z \in \mathcal Z\), we define its entropic Sinkhorn support as
\begin{equation}
K^\nu_\varepsilon(z)
:=
\int_{\mathcal Z}
\exp\!\left(-\frac{d(z,u)}{\varepsilon}\right)d\nu(u),
\qquad z\in\mathcal Z .
\end{equation}
This quantity aggregates the amount of in-distribution mass that can support \(z\), where nearby mass contributes more strongly and the temperature \(\varepsilon\) controls the softness of the transport neighborhood. A large value of \(K^\nu_\varepsilon(z)\) indicates that \(z\) is well supported by the reference measure, whereas a small value indicates weak support under the Sinkhorn transport geometry.

\begin{definition}[Sinkhorn Outlier Energy]
\label{def:outlier_energy}
Let $\nu \in \mathcal{P}(\mathcal{Z})$ be a reference measure on the latent
representation space $\mathcal{Z}$, let
$d : \mathcal{Z} \times \mathcal{Z} \to \mathbb{R}_+$ be a ground cost, and let
$\varepsilon>0$. For any $z \in \mathcal{Z}$ with
$0<K^\nu_\varepsilon(z)<\infty$, we define the Sinkhorn outlier energy of $z$ as
\begin{equation}
E^\nu_\varepsilon(z)
:=
-\varepsilon \log K^\nu_\varepsilon(z).
\end{equation}
\end{definition}\vspace{-8pt}
Sinkhorn Outlier Energy (SOE) is the transport-support energy used inside SBOG: larger values indicate candidate points that are more expensive to support from the in-distribution reference measure under the chosen ground geometry. SBOG uses SOE to define boundary contours for anchor selection and candidate acceptance, and its relation to empirical kernel support is clarified in Lemma~\ref{lem:soe_support_cost}.

\begin{lemma}[Entropic support cost of Sinkhorn outlier energy]
\label{lem:soe_support_cost}
For any $z\in Z$ such that $0<K^\nu_\epsilon(z)<\infty$, given the Sinkhorn outlier
energy $E^\nu_\epsilon(z)$, in particular, for the empirical reference measure
$\hat\nu_n=\frac1n\sum_{i=1}^n\delta_{z_i}$ with distinct atoms, we have
\[
E^{\hat\nu_n}_\epsilon(z)
=
\min_{w\in\Delta_n}
\left\{
\sum_{i=1}^n w_i d(z,z_i)
+
\epsilon\sum_{i=1}^n w_i\log(nw_i)
\right\},
\]
where $\Delta_n=\{w\in\mathbb R_+^n:\sum_{i=1}^n w_i=1\}$ with the minimizer being $w_i^\star(z) =\frac{\exp\!\left(-d(z,z_i)/\epsilon\right)}
{\sum_{j=1}^n \exp\!\left(-d(z,z_j)/\epsilon\right)}$.
\end{lemma}\vspace{-10pt}
\textbf{Relation to kernel support.} Lemma~\ref{lem:soe_support_cost} shows that, for the empirical reference measure $\hat{\nu}n=\frac{1}{n}\sum{i=1}^n\delta_{z_i}$, SOE admits a kernel-support form. Only in the special case of uniform empirical support with squared Euclidean cost does SOE reduce to a monotone transform of an unnormalized RBF-KDE score. Beyond this restricted setting, SOE is not equivalent to KDE, and our experiments further show that this transport-based boundary criterion yields stronger outlier anchors than KDE-based alternatives. In SBOG, we use this connection to reinterpret kernel support as an entropic transport support cost:
$E^\nu_\epsilon(z)=-\epsilon\log K^\nu_\epsilon(z)$ measures how costly it is, under the chosen ground geometry, for the in-distribution reference measure to assign support to $z$. Inside SBOG, this energy defines smooth boundary contours for anchor selection and candidate acceptance, while semantic feasibility and decoding constraints turn these contours into valid outlier candidates. The construction is modality-agnostic in latent space and, through the shared Gibbs normalizer in the Sinkhorn-DRO dual, links boundary selection to a robust downstream decision view. 

\begin{wrapfigure}{r}{0.56\textwidth}
\vspace{-25pt}
\begin{minipage}{0.56\textwidth}
\begin{algorithm}[H]
\caption{Sinkhorn Boundary Outlier Generation (SBOG)}
\label{alg:outlier_sampling}
\small
\begin{algorithmic}[1]
\Require $D_{\mathrm{in}}=\{(x_i,y_i)\}_{i=1}^n$, encoder $h_\theta$, decoder $p(x\mid v)$
\Require class anchors $\{t_c\}_{c=1}^C$, semantic floor $\tau_{\mathrm{sem}}\ge 0$
\Require $\varepsilon,\tau_E,\sigma>0$, boundary anchors $A$, proposals $M$, samples $N_{\mathrm{ood}}$, projection $\Pi_{\mathcal Z}$, $\widehat E(\cdot):=E^{\widehat\nu_n}_\varepsilon(\cdot)$
\State $Z_n \gets \{z_i=z(x_i;\theta)\}_{i=1}^n$, \quad
$\widehat\nu_n \gets \frac{1}{n}\sum_{i=1}^n \delta_{z_i}$
\State let $\mathcal I_A$ be the indices of the top-$A$ values of $\{\widehat E(z_i)\}_{i=1}^n$
\State $D_{\mathrm{ood}} \gets \emptyset$
\For{$k=1,\dots,N_{\mathrm{ood}}$}
    \Repeat
        \State sample $i \sim \mathrm{Unif}(\mathcal I_A)$
        \State $v_j \gets \Pi_{\mathcal Z}(z_i+\eta_j)$, \ $\eta_j\sim\mathcal N(0,\sigma^2 I)$, \ $j\in[M]$
        \State $\mathcal J \gets \{j:\widehat E(v_j)\ge \tau_E \ \text{and}\ \cos(v_j, t_{y_i})\ge \tau_{\mathrm{sem}}\}$
    \Until{$\mathcal J \neq \emptyset$}
    \State $v \gets \arg\max_{j\in\mathcal J} \widehat E(v_j)$, \quad $x_{\mathrm{ood}} \sim p(x\mid v)$
    \State $D_{\mathrm{ood}} \gets D_{\mathrm{ood}} \cup \{x_{\mathrm{ood}}\}$
\EndFor
\State \Return $D_{\mathrm{ood}}$
\end{algorithmic}
\end{algorithm}
\end{minipage}
\vspace{-18pt}
\end{wrapfigure}
Given the empirical latent reference set $Z_n$, SBOG first computes $\widehat E(z_i):=E^{\widehat\nu_n}_\varepsilon(z_i)$ for all embeddings $z_i\in Z_n$ and selects the top-energy embeddings as boundary anchors. Around each selected anchor, SBOG draws local latent proposals by adding isotropic Gaussian perturbations and projecting them back to $\mathcal Z$ when needed. A proposal $v$ is feasible only if it crosses the SOE support-cost contour $\widehat E(v)\ge \tau_E$ and satisfies the semantic feasibility constraint $\cos(v,t_{y_i})\ge \tau_{\mathrm{sem}}$, where $t_{y_i}$ is the class anchor associated with the source embedding $z_i$. Among the feasible proposals, SBOG selects the candidate with the largest SOE, the most transport-expensive point that remains semantically anchored, and then decodes it into input space via $x_{\mathrm{ood}}\sim p(x\mid v)$. Unlike purely statistical density-rejection or hard $k$-NN-based methods \cite{npos,du2023dream}, SBOG does not rely on a low-support criterion alone. It uses a smooth transport-support contour for boundary selection and explicitly couples this boundary criterion with semantic anchoring and generative validity.

The preceding sampler uses SOE as a model-agnostic boundary energy. We next show that this energy is not an isolated support heuristic, but the support-normalization term induced by the same Sinkhorn-DRO geometry. For a radius $r_{\mathrm{DRO}}>0$, the Sinkhorn uncertainty set around $\mu$ defines
\begin{equation}
\sup_{P \in \mathcal P(\mathcal Z):\ W^\nu_\varepsilon(\mu,P)\le r_{\mathrm{DRO}}}
\; \mathbb E_{u\sim P}[f(u)],
\end{equation}
for a latent loss or score function $f$. The next result explains why the SOE threshold and energy-maximizing feasible proposal in SBOG correspond to transport support-boundary acquisition.
\begin{proposition}[Sinkhorn outlier energy in the Sinkhorn-DRO dual]
\label{thm:outlier-in-dRO}
Under the definitions above, for any measurable $f:\mathcal Z\to\mathbb R$ whose
log-moment is finite for every $\lambda>0$, the Sinkhorn-DRO dual objective admits the following
decomposition:
\begin{align}
V_D(\lambda)
&=
\lambda r_{\mathrm{DRO}}
+
\lambda\varepsilon
\mathbb E_{z\sim\mu}
\left[
\log
\mathbb E_{u\sim Q_{\varepsilon,z}^{\nu}}
\exp\!\left(
\frac{f(u)}{\lambda\varepsilon}
\right)
\right]
\nonumber\\
&=
\lambda r_{\mathrm{DRO}}
+
\mathbb E_{z\sim\mu}
\left[
\lambda\varepsilon
\log
\int_{\mathcal Z}
\exp\!\left(
\frac{f(u)}{\lambda\varepsilon}
-
\frac{d(z,u)}{\varepsilon}
\right)
d\nu(u)
+
\lambda E^\nu_\varepsilon(z)
\right],
\label{eq:VD-decomp}
\end{align}
where $Q_{\varepsilon,z}^{\nu}$ is the Gibbs kernel and
$E^\nu_\varepsilon(z)= -\varepsilon\log K^\nu_\varepsilon(z)$ is the Sinkhorn outlier energy from
Definition~\ref{def:outlier_energy}.  Thus, the Sinkhorn-DRO dual objective contains the outlier
energy term $\lambda\mathbb E_{z\sim\mu}\!\left[E^\nu_\varepsilon(z)\right]$.
\end{proposition}
Proposition~\ref{thm:outlier-in-dRO} clarifies the role of SOE in SBOG. The
decomposition shows that SOE is the support-normalization cost induced by the
Sinkhorn uncertainty geometry. SBOG uses this term in a model-agnostic way:
the threshold $\widehat E(v)\ge \tau_E$ selects candidates beyond a transport
support-cost contour, and the final maximization chooses the most
transport-expensive candidate among the semantically feasible proposals. Thus,
the DRO connection justifies the support-boundary acquisition rule used by
SBOG, rather than claiming that the submitted algorithm solves a downstream task-specific DRO problem. If a downstream model is available, one can augment the acquisition with a task loss or uncertainty score using the same Gibbs kernel, and we leave this task-aware variant as a direct extension. Since SBOG relies on the empirical SOE threshold $\widehat E(v)\ge\tau_E$ for anchor selection and proposal acceptance, we next analyze the finite-sample stability of this boundary decision.
\begin{theorem}[Finite-sample stability of SOE thresholding]
\label{thm_soe}
For a fixed finite candidate pool $\mathcal V \subset \mathcal Z$, define $\underline K:=\inf_{v \in \mathcal V}K^{\bar{\nu}}_\varepsilon(v) > 0$, $\Delta_E=\inf_{v \in \mathcal V}\left|E^{\bar{\nu}}_\varepsilon(v) - \tau_E\right| > 0$, and $\tau_n(\delta)=\sqrt{{\log(2|\mathcal V|/\delta)}/{2n}}$. If $\tau_n(\delta) \le \underline K/2$, then with probability at least $1-\delta$,
\[
\sup_{v \in \mathcal V}
\left|
E^{\widehat{\nu}_n}_\varepsilon(v)
-
E^{\bar{\nu}}_\varepsilon(v)
\right|
\le
\frac{2\varepsilon}{\underline K}\tau_n(\delta).
\]
Consequently, if $\tau_n(\delta)\le\min\!\left\{\frac{\underline K}{2},\frac{\underline K \Delta_E}{4\varepsilon}\right\}$, then with probability at least $1-\delta$, $\mathbf 1\{E^{\widehat{\nu}_n}_\varepsilon(v) \ge \tau_E\}=\mathbf 1\{E^{\bar{\nu}}_\varepsilon(v) \ge \tau_E\}$ for any $v \in \mathcal V$. In particular, it suffices that
\[
n \ge
\max\!\left\{
\frac{2}{\underline K^2}\log\frac{2|\mathcal V|}{\delta},
\;
\frac{8\varepsilon^2}{\underline K^2\Delta_E^2}
\log\frac{2|\mathcal V|}{\delta}
\right\}.
\]
\end{theorem}

For a fixed candidate pool $\mathcal V$, the sufficient sample size scales as $O\!\left(\frac{\varepsilon^2}{\underline K^2\Delta_E^2}\log\frac{|\mathcal V|}{\delta}\right)$ up to the $\underline K^{-2}\log(|\mathcal V|/\delta)$ normalizer condition. Thus, the dependence on the latent geometry, the Sinkhorn temperature, and the placement of the candidate pool is captured by $\underline K$ and by the attainable energy margin $\Delta_E$. In particular, if $\varepsilon$ is too small or the candidate pool lies in extremely weakly supported regions, then $\underline K$ may be small and the bound can become loose. We also analyze the sample complexity of neighborhood-based methods in Appendix~\ref{app_thm_neighbor}, highlighting the advantages of SOE.

We next give a theorem for the consistency of the generated distribution. Write $\widehat E_n=E_\varepsilon^{\widehat\nu_n}$ and
$E=E_\varepsilon^{\bar\nu}$.
Let $Q_n^{\mathrm{SBOG}}$ and $Q_{n,\mathrm{SOE}}^\star$
be the class--latent output distributions of Algorithm \ref{alg:outlier_sampling} using $\widehat E_n$ and $E$, respectively, on the same observed bank. Denote their top-$A$ anchor sets by $\widehat I_{A,n}$ and $I_A^\star$, and their energy thresholds by $\widehat\tau_n$ and $\tau_n^\star$. All remaining components and tie-breaking rules are shared. Define $e_n=\sup_{v\in\mathcal Z}|\widehat E_n(v)-E(v)|$ and $a_n=\frac{|\widehat I_{A,n}\triangle I_A^\star|}{2A}$,where $\triangle$ denotes symmetric difference. In one population-SOE proposal round, let $S=\{j:\cos(V_j,t_{y_I})\ge\tau_{\mathrm{sem}}\}$ and $J^\star=\{j\in S:E(V_j)\ge\tau_n^\star\}$. Let $\Delta_n$ be the gap between the two largest energies in $J^\star$, with $\Delta_n=\infty$ when $|J^\star|<2$. Conditional on the observed bank, define
\[b_n(t)=\Pr\!\left(\min_{j\in S}|E(V_j)-\tau_n^\star|\le2t\ \text{or}\ \Delta_n\le2t \right),
\]
and $p_n^\star=\Pr(J^\star\ne\varnothing),
$ with $\min\varnothing=\infty$.

\begin{theorem}[Population-SOE consistency]
\label{thm:sbog_population_consistency}
Suppose $|\widehat\tau_n-\tau_n^\star|\le e_n$,
$p_n^\star\ge p_0>0$, and $a_n+b_n(e_n)<p_0$.
Then the empirical sampler terminates almost surely, and
\begin{equation}
\label{eq:sbog_population_tv}
d_{\mathrm{TV}}\!\left(
Q_n^{\mathrm{SBOG}},Q_{n,\mathrm{SOE}}^\star
\right)
\le
\min\!\left\{1,\frac{2\{a_n+b_n(e_n)\}}{p_0}\right\}
=:\zeta_n,
\end{equation}
where $d_{\mathrm{TV}}(P,Q)=\sup_B|P(B)-Q(B)|$. The same bound holds after applying a shared conditional decoder. If $e_n\to0$ and $a_n+b_n(e_n)\to0$ in probability over the bank, and the threshold and acceptance conditions hold with probability tending to one, then $d_{\mathrm{TV}}(Q_n^{\mathrm{SBOG}},Q_{n,\mathrm{SOE}}^\star)\to0$ in probability.
\end{theorem}

The proof is given in Appendix~\ref{app:sbog_population_consistency}. This result concerns consistency with the population SOE criterion, not alignment with an independently defined target boundary.

\subsection{Sinkhorn-DRO Alignment for Boundary-Aware Representations}
\label{sec:robust_alignment}

SBOG relies on a latent ground cost whose local neighborhoods reflect both geometric proximity and semantic similarity. If the representation is trained only with pointwise alignment logits, the resulting distances may separate classes at training points but still give poorly calibrated boundary geometry for outlier generation. We therefore use the same Sinkhorn-DRO principle to robustify the alignment objective: instead of matching each sample only to its class prototype at a point, we require the alignment to remain stable under a local Sinkhorn transport neighborhood. Concretely, consider the standard alignment objective
\begin{equation}
\mathcal{L}
=
\mathbb{E}_{(x,y)\sim \mathcal{D}}
\left[
-\log
\frac{\exp\!\left(T(y)^\top z(x;\theta) / t\right)}
{\sum_{c=1}^{C}\exp\!\left(T(c)^\top z(x;\theta) / t\right)}
\right],
\end{equation}
where $z(x;\theta)=h_\theta(x)/\|h_\theta(x)\|_2$ and $\{T(c)\}_{c=1}^C$ are class prototypes. Since these logits are purely pointwise, they may not adequately reflect the boundary geometry of the in-distribution, where informative outlier signals tend to concentrate.

We incorporate local Sinkhorn geometry through a reference measure $\nu_x\in\mathcal P(\mathcal X)$ centered around $x$. Let $W_{\varepsilon}^{\nu_x}$ denote the normalized Sinkhorn OT discrepancy with reference measure $\nu_x$ and regularization $\varepsilon>0$. For class $c$, define $\ell_{\mathrm{align}}(\tilde{x},c;\theta)=-\frac{T(c)^\top z(\tilde{x};\theta)}{t}$ and $z(\tilde{x};\theta)=\frac{h_\theta(\tilde{x})}{\|h_\theta(\tilde{x})\|_2}$. We replace the nominal logit by the Sinkhorn-DRO robust logit
\[
V_c(x;\theta)
:=
-
\sup_{q\in\mathcal P(\mathcal X):W^{\nu_x}_\varepsilon(\delta_x,q)\le \rho}
\mathbb{E}_{\tilde{x}\sim q}
\!\left[
\ell_{\mathrm{align}}(\tilde{x},c;\theta)
\right],
\]
where $\delta_x$ is the point mass at $x$, and $r_{\mathrm{DRO}}>0$ is the local DRO radius. By Lemma~\ref{lem:sinkhorndro}, this is equivalently given by the corresponding one-dimensional dual with
$f(\tilde{x})=\ell_{\mathrm{align}}(\tilde{x},c;\theta)$, namely,
\[
V_c(x;\theta)
=
-\inf_{\lambda>0}
\left\{
\lambda r_{\mathrm{DRO}}
+
\lambda\varepsilon
\log
\mathbb{E}_{\tilde{x}\sim Q^{\nu_x}_{\varepsilon,x}}
\left[
\exp\!\left(
\frac{\ell_{\mathrm{align}}(\tilde{x},c;\theta)}
{\lambda\varepsilon}
\right)
\right]
\right\}.
\]
Using $\{V_c(x;\theta)\}_{c=1}^C$ in place of the nominal similarity logits yields
\begin{equation}
\mathcal{L}_{\mathrm{rob}}
=
\mathbb{E}_{(x,y)\sim\mathcal{D}}
\left[
-\log
\frac{\exp(V_y(x;\theta))}
{\sum_{c=1}^C \exp(V_c(x;\theta))}
\right].
\end{equation}

\begin{remark}
If $r_{\mathrm{DRO}}=0$ and $\nu_x=\delta_x$, then $Q^{\nu_x}_{\varepsilon,x}=\delta_x$ and $V_c(x;\theta)=
-\ell_{\mathrm{align}}(x,c;\theta)=T(c)^\top z(x;\theta)/t$. Therefore $\mathcal{L}_{\mathrm{rob}}$ reduces to the standard alignment loss. More generally, when $\nu_x$ is chosen as a local neighborhood measure and $d_{\mathcal X}$ is a metric, the robust logit aggregates local evidence through a log-sum-exp tilt, recovering neighbor- or margin-based variants as limiting cases.
\end{remark}
This alignment module makes the learned representation geometry more compatible with SBOG by using the same local Sinkhorn neighborhoods in both robust alignment and SOE. The robust logits aggregate local perturbations through the Sinkhorn transport geometry, while SOE uses the corresponding Gibbs kernel to measure support for boundary generation.
\section{Experiments}
Since synthetic outliers do not have a canonical target distribution or a universally accepted fidelity metric, we evaluate them through their downstream utility, following standard practice in the outlier-synthesis literature \cite{du2023dream,darban2025carla}. Specifically, we borrow detector-side benchmarks as a practical and task-relevant proxy: better synthesized outliers should better cover informative regions outside the in-distribution support, and therefore lead to a clearer boundary between inliers and outliers therefore improved discrimination performance \cite{hendrycks2018deep}. All experiments are conducted on a machine equipped with 13th Gen Intel(R) Core(TM) i9-13900HX CPU (24 cores) and NVIDIA A100 GPU (40GB). The experimental details are in Appendix Section \ref{app:ts-details}. The code is available at \url{https://github.com/hSun08/SBOG}.

\subsection{Time Series Generation}\label{sec:ts_gen}
We evaluate SBOG on five standard multivariate time-series anomaly detection benchmarks, including SMD \cite{smd}, MSL \cite{hundman2018detecting}, SMAP, SWaT \cite{swat}, and WADI \cite{wadi} and  baselines of COUTA \cite{xu2024calibrated}, Ts2Vec \cite{yue2022ts2vec}, Dcdetector  \cite{yang2023dcdetector}, TranAD \cite{tuli2022tranad}, AnomalyTransformer \cite{xu2022anomaly}, ScatterAD \cite{yinscatterad}, THOC \cite{NEURIPS2020_97e401a0} and Sub-Adjacent Transformer \cite{subadjacent}. We use F1 and AU-PR as the main evaluation metrics for each method. To avoid overestimating detection performance~\cite{kim2022towards}, our main tables use the stricter point-wise evaluation without Point Adjustment. 
\begin{table}[h]
\centering\vspace{-10pt}
\caption{Comparison results on SMD, MSL, SMAP, SWaT, and WADI datasets. \textbf{Bold} indicates the best result among the synthesis baselines, and $\dagger$ indicates the best result among all baselines.}
\label{tab:ts_main_results}
\setlength{\tabcolsep}{6pt} 
\renewcommand{\arraystretch}{1.15}
\resizebox{1\textwidth}{!}{
\begin{tabular}{lcccccccccccc}
\toprule
\multirow{2}{*}{Methods} 
& \multicolumn{2}{c}{SMD} 
& \multicolumn{2}{c}{MSL} 
& \multicolumn{2}{c}{SMAP} 
& \multicolumn{2}{c}{SWaT} 
& \multicolumn{2}{c}{WADI} 
& \multicolumn{2}{c}{Average} \\
\cmidrule(lr){2-3}
\cmidrule(lr){4-5}
\cmidrule(lr){6-7}
\cmidrule(lr){8-9}
\cmidrule(lr){10-11}
\cmidrule(lr){12-13}
& AU-PR & F1 & AU-PR & F1 & AU-PR & F1 & AU-PR & F1 & AU-PR & F1 & AU-PR & F1\\
\midrule
AnomalyTransformer  \cite{xu2022anomaly}                         & 0.082 & 0.137 & 0.129 & 0.218 & 0.203 & 0.307 & 0.375 & 0.531 & 0.064 & 0.145 & 0.171 & 0.268 \\
DCdetector \cite{yang2023dcdetector}                                   & 0.044 & 0.086 & 0.115 & 0.204 & 0.131 & 0.225 & 0.401 & 0.537 & 0.059 & 0.109 & 0.150 & 0.232 \\
TranAD \cite{tuli2022tranad}                                       & 0.389 & 0.450 & 0.204 & 0.314 & 0.196 & 0.308 & $0.464^{\dagger}$ & $0.684^{\dagger}$ & 0.040 & 0.109 & 0.259 & 0.373 \\
COUTA \cite{xu2024calibrated}                                       & $0.409^{\dagger}$ & $0.461^{\dagger}$ & 0.321 & 0.403 & 0.309 & 0.382 & 0.303 & 0.535 & 0.207 & 0.266 & 0.310 & 0.410 \\ 
TS2Vec  \cite{yue2022ts2vec}                                      & 0.113 & 0.186 & 0.145 & 0.253 & 0.165 & 0.276 & 0.337 & 0.532 & 0.057 & 0.114 & 0.163 & 0.272 \\
ScatterAD  \cite{yinscatterad}                                   & 0.300 & 0.360 & 0.144 & 0.234 & 0.143 & 0.251 & 0.347 & 0.532 & 0.053 & 0.121 & 0.198 & 0.299 \\
THOC \cite{NEURIPS2020_97e401a0}                                         & 0.321 & 0.381 & 0.227 & 0.322 & 0.235 & 0.333 & 0.301 & 0.538 & 0.053 & 0.125 & 0.227 & 0.340 \\
Sub-Adjacent Transformer  \cite{subadjacent}                    & 0.185 & 0.269 & 0.295 & 0.408 & 0.171 & 0.269 & 0.384 & 0.533 & $0.458^{\dagger}$ & $0.571^{\dagger}$ & 0.299 & 0.410 \\
\hline
\textit{Synthesis baselines}\\
CARLA \cite{darban2025carla}                                        & 0.236 & 0.296 & 0.337 & 0.420 & 0.304 & 0.378 & 0.402 & 0.606 & 0.058 & 0.109 & 0.267 & 0.362 \\
CARLA-RandomNeg                               & 0.219 & 0.285 & 0.334 & 0.431 & 0.274 & 0.352 & 0.319 & 0.531 & 0.028 & 0.102 & 0.235 & 0.340 \\
CARLA-$k$-NNNeg                                  & 0.371 & 0.434 & 0.250 & 0.350 & 0.390 & 0.451 & 0.402 & 0.604 & 0.045 & 0.109 & 0.291 & 0.390 \\
CARLA-KDENeg                                  &0.258 & 0.348 & 0.351 & 0.445 & 0.391 & 0.465 & 0.350 & 0.578 & 0.049 & 0.129 & 0.280 & 0.393\\
\rowcolor{gray!15} 
CARLA-SBOG                                      & \textbf{0.413} & \textbf{0.471} & $\textbf{0.352}^{\dagger}$ & $\textbf{0.463}^{\dagger}$ & $\textbf{0.449}^{\dagger}$ & $\textbf{0.515}^{\dagger}$ & \textbf{0.369} & \textbf{0.625} & \textbf{0.063} & \textbf{0.133} & $\textbf{0.312}^{\dagger}$ & $\textbf{0.432}^{\dagger}$ \\
\bottomrule
\end{tabular}\vspace{-25pt}
}
\end{table}

\vspace{-5pt}
The upper part of Table \ref{tab:ts_main_results} reports general time-series anomaly detection baselines, while CARLA~\cite{darban2025carla} is the only compatible framework that explicitly uses synthetic negatives for time-series anomaly detection. We therefore use CARLA as a controlled testbed with only the negative-sample construction rule changed. Specifically, we compare random negatives, $k$-NN based sparse negatives, KDE-based support negatives, and SBOG negatives. It is worth noting that the robust alignment module is not used in the time-series experiments because these benchmarks do not contain semantic class anchors or prototype labels that define logits $T(c)^\top z$.  Under this matched setting, CARLA-SBOG achieves the best average AU-PR and F1 among all CARLA variants. This indicates that the SBOG boundary rule produces more useful negatives for downstream anomaly detection.

\begin{figure}[h]
    \centering\vspace{-10pt}
    \includegraphics[width=1.0\linewidth]{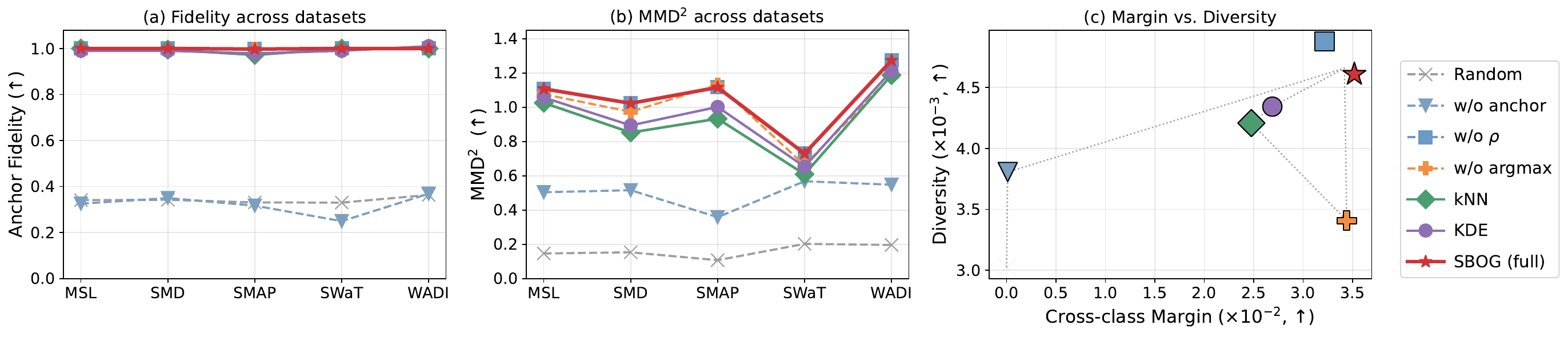}
    \caption{Component ablation of the time-series sampler.}
    \label{fig:ts-component}\vspace{-10pt}
\end{figure}
To further examine whether the gains come from the SOE sampler itself, we ablate the main mechanisms of Algorithm~1, including high-energy anchor selection, energy-threshold acceptance, and argmax candidate selection, and additionally compare with $k$-NN and KDE support-based variants. All variants use the same trained encoder, so the comparison isolates the effect of the sampler and boundary scoring rule. Panels (a) and (b) show that SBOG maintains high anchor fidelity while inducing a stronger distributional shift as measured by MMD$^2$, whereas random sampling and removing anchor selection reduce fidelity. Panel (c) further shows that SBOG achieves a better margin--diversity trade-off \begin{wraptable}{r}{0.4\linewidth}
\vspace{-10pt}
\centering
\caption{Sensitivity to the threshold $\tau_E$}\vspace{-8pt}
\label{tab:tau_sensitivity}
\small
\setlength{\tabcolsep}{3pt}
\renewcommand{\arraystretch}{1.05}
\begin{tabular}{@{}ccccc@{}}
\toprule
$\tau_E$ & Avg AUPR & Avg F1 & Gap & Fill \\
\midrule
$0.50$              & 0.279 & 0.387 & 0.225 & 1.000 \\
$0.70$              & 0.295 & 0.409 & 0.224 & 1.000 \\
$0.90$              & 0.310 & 0.430 & 0.217 & 1.000 \\
$\mathbf{0.95}$     & $\mathbf{0.312}$ & $\mathbf{0.432}$ & 0.126 & 1.000 \\
$0.98$              & 0.305 & 0.422 & 0.109 & 0.999 \\
$0.99$            & 0.277 & 0.384 & 0.079 & 0.945 \\
\bottomrule
\end{tabular}
\vspace{-18pt}
\end{wraptable} than the ablated and support-score variants: it moves generated negatives farther along the boundary while retaining high intra-negative diversity. This suggests that selecting the most energetic feasible candidate improves boundary informativeness beyond generic low-support scoring.

We sweep the threshold $\tau_E$ using the same trained encoder and report Avg AUPR and F1 together with two proposal diagnostics. Gap measures the average margin $E(\mathrm{neg})-\tau_E$ of accepted negatives, and Fill is the fraction of anchors that yield at least one feasible proposal. Table~\ref{tab:tau_sensitivity}
shows that performance is stable over a broad middle range of thresholds, while overly high thresholds reduce both the margin and feasible-anchor coverage. This is consistent with Theorem~\ref{thm_soe}: reliable thresholding requires a nontrivial energy margin and a well-conditioned support normalizer.

\subsection{Image Generation}
For a more comprehensive evaluation, we also apply our method to image outlier generation. We follow the experimental protocol used in recent image outlier synthesis works \cite{npos,du2023dream,gao2025good,liao2025bood} and the standard benchmark setting in \cite{huang2021mos} for evaluation. For fair comparison, we adopt ResNet-34 \cite{he2016deep} as the ID classifier and CIFAR-100 and ImageNet-100 \cite{deng2009imagenet} as the in-distribution datasets. Following prior baselines, we generate OOD samples with Stable Diffusion v1.5. On ImageNet-100, we use iNaturalist \cite{van2018inaturalist}, SUN \cite{xiao2010sun}, Places \cite{zhou2017places}, and Textures \cite{cimpoi2014describing} as test datasets. On CIFAR-100, we use SVHN \cite{netzer2011reading}, Places365 \cite{zhou2017places}, LSUN-R \cite{yu2015lsun}, iSUN \cite{xu2015turkergaze}, and Textures \cite{cimpoi2014describing} as test datasets.

We compare our method against representative baselines for synthesis methods about outlier detection, including GAN \cite{lee2018training}, VOS \cite{du2022towards}, NPOS \cite{npos}, DreamOOD \cite{du2023dream}, BOOD \cite{liao2025bood}, FodFoM \cite{chen2024fodfom}, and GOOD \cite{gao2025good}. Following the standard evaluation protocol used by these baselines, we assess the quality of the
generated data using FPR95, the false positive rate when the true positive rate for ID samples is $95\%$ and AUROC, the area under the receiver operating characteristic curve. The results on ImageNet-100 and CIFAR-100 are reported in Table \ref{tab:img100_ood} and Table \ref{tab:cifar100_ood}, respectively. For methods without official released code, we report the summarized results provided in recent benchmark papers \cite{gao2025good,liao2025bood}. Methods marked with an asterisk (*) do not have publicly available official implementations and are therefore not included in the formal ranking. In the tables, \textbf{bold} indicates the best result and \underline{underline} indicates the second-best result.

\begin{table*}[h]
\centering
\small
\caption{Comparative evaluation of outlier detection performance on ImageNet-100 as the ID dataset.}\vspace{-5pt}
\label{tab:img100_ood}
\resizebox{0.97\textwidth}{!}{
\begin{tabular}{lcccccccccc}
\toprule
\multirow{2}{*}{Methods} & \multicolumn{2}{c}{INATURALIST} & \multicolumn{2}{c}{PLACES} & \multicolumn{2}{c}{SUN} & \multicolumn{2}{c}{TEXTURES} & \multicolumn{2}{c}{Average}  \\
\cmidrule(lr){2-3} \cmidrule(lr){4-5} \cmidrule(lr){6-7} \cmidrule(lr){8-9} \cmidrule(lr){10-11}
& FPR95$\downarrow$ & AUROC$\uparrow$ & FPR95$\downarrow$ & AUROC$\uparrow$ & FPR95$\downarrow$ & AUROC$\uparrow$ & FPR95$\downarrow$ & AUROC$\uparrow$ & FPR95$\downarrow$ & AUROC$\uparrow$  \\
\hline
GAN  \cite{lee2018training}      & 83.10 & 71.35 & 83.20 & 69.85 & 84.40 & 67.56 & 91.00 & 59.16 & 85.42 & 66.98  \\
VOS  \cite{du2022towards}       & 43.00 & 93.77 & 47.60 & 91.77 & 39.40 & 93.17 & 66.10 & 81.42 & 49.02 & 90.03  \\
NPOS \cite{npos}    & 53.84 & 86.52 & 59.66 & 83.50 & 53.54 & 87.99 & \textbf{8.98}  & \textbf{98.13} & 44.00 & 89.04  \\
Dream-OOD \cite{du2023dream}   & 24.10 & 96.10 & 39.87 & 93.11 & 36.88 & 93.31 & 53.99 & 85.56 & 38.76 & 92.02  \\
NCIS  \cite{doorenbos2024non}   & 20.70 & 96.56 & 34.60 & 94.07 & 35.43 & 94.13 & 44.83 & 88.50 & 33.89 & 93.32  \\
BOOD* \cite{liao2025bood}       & 18.33 & 96.74 & 33.33 & 94.08 & 37.92 & 93.52 & 51.88 & 85.41 & 35.37 & 92.44  \\
GOOD* \cite{gao2025good}       & 9.22  & 97.61 & 24.79 & 94.82 & 17.20 & 96.10 & 19.51 & 96.60 & 17.68 & 96.30 \\
\rowcolor{gray!15}
SBOG        & \beststd{9.44}{0.73} & \beststd{98.09}{0.10} & \secondstd{14.40}{1.01} & \secondstd{97.09}{0.16} & \secondstd{13.22}{0.92} & \secondstd{97.61}{0.11} & \std{24.26}{1.06} & \std{94.82}{0.23} & \secondstd{15.33}{0.46} & \secondstd{96.90}{0.08} \\
\rowcolor{gray!15}
SBOG (robust) & \secondstd{9.89}{0.76} & \secondstd{97.99}{0.10} & \beststd{13.24}{0.85} & \beststd{97.21}{0.15} & \beststd{11.81}{0.99} & \beststd{97.64}{0.11} & \secondstd{24.06}{1.06} & \secondstd{95.05}{0.22} & \beststd{14.75}{0.46} & \beststd{96.97}{0.08}\\
\bottomrule
\end{tabular}
}\vspace{-10pt}
\end{table*}
The results show that SBOG achieves near state-of-the-art outlier detection performance on both CIFAR-100 and ImageNet-100. More importantly, SOE improves over several representative statistical support-based synthesis methods, especially VOS \cite{du2022towards} and Dream-OOD~\cite{du2023dream}, which sample virtual outliers from low-likelihood regions, and NPOS~\cite{npos},  which is the cloest non-parametric density estimation baseline. The results support the central design of SBOG: SOE identifies informative weak-support boundary regions, while semantic anchoring controls class drift. The SBOG (robust) variant uses the Sinkhorn-DRO alignment module from Section \ref{sec:robust_alignment}. Its small but consistent improvement suggests that calibrating the latent representation with the same Sinkhorn geometry used by SOE can improve the boundary candidates generated by SBOG.

\begin{table}[t]
\centering
\caption{Comparative evaluation of outlier detection performance on CIFAR-100 as the ID dataset.}
\label{tab:cifar100_ood}
\resizebox{0.98\textwidth}{!}{
\begin{tabular}{lcccccccccccc}
\toprule
\multirow{2}{*}{Methods} 
& \multicolumn{2}{c}{SVHN} 
& \multicolumn{2}{c}{PLACES365} 
& \multicolumn{2}{c}{LSUN-R} 
& \multicolumn{2}{c}{iSUN} 
& \multicolumn{2}{c}{TEXTURES} 
& \multicolumn{2}{c}{Average} \\
\cmidrule(lr){2-3} \cmidrule(lr){4-5} \cmidrule(lr){6-7} \cmidrule(lr){8-9} \cmidrule(lr){10-11} \cmidrule(lr){12-13}
& FPR95$\downarrow$ & AUROC$\uparrow$
& FPR95$\downarrow$ & AUROC$\uparrow$
& FPR95$\downarrow$ & AUROC$\uparrow$
& FPR95$\downarrow$ & AUROC$\uparrow$
& FPR95$\downarrow$ & AUROC$\uparrow$
& FPR95$\downarrow$ & AUROC$\uparrow$ \\
\hline 
GAN \cite{lee2018training} & 89.45 & 66.95 & 88.75 & 66.76 & 82.35 & 75.87 & 83.45 & 73.49 & 92.80 & 62.99 & 87.36 & 69.21 \\
VOS \cite{du2022towards} & 78.50 & 73.11 & 84.55 & 75.85 & 59.05 & 85.72 & 72.45 & 82.66 & 75.35 & 80.08 & 73.98 & 79.48 \\
NPOS \cite{npos} & 11.14 & 97.84 & 79.08 & 71.30 & 56.27 & 82.43 & 51.72 & 85.48 & 35.20 & 92.44 & 46.68 & 85.90 \\
DreamOOD \cite{du2023dream} & 58.75 & 87.01 & 70.85 & 79.94 & 24.25 & 95.23 & 1.10 & 99.73 & 46.60 & 88.82 & 40.31 & 90.15 \\
BOOD*  \cite{liao2025bood} & 5.42 & 98.43 & 40.55 & 90.76 & 2.06 & 99.25 & 0.22 & 99.91 & 5.10 & 98.74 & 10.67 & 97.42\\
FodFoM \cite{chen2024fodfom}          & 33.19 & 94.02 & \textbf{42.30} & \textbf{90.68} & 28.24 & 95.09 & 33.06 & 94.45 & 35.44 & 93.38 & 34.45 & 93.52 \\
GOOD*  \cite{gao2025good} & 19.37 & 95.34 & 65.85 & 82.04 & 18.82 & 94.15 & 0.81  & 99.59 & 23.47 & 95.98 & 25.66 & 93.34 \\
\rowcolor{gray!15}
SBOG   & \beststd{5.38}{0.64}  & \secondstd{98.67}{0.10} & \secondstd{47.48}{1.70} & \secondstd{89.70}{0.34} & \secondstd{3.48}{0.52}  & \beststd{99.13}{0.06} & \beststd{0.10}{0.07} & \beststd{99.95}{0.01} & \secondstd{5.74}{0.61}  & \secondstd{98.69}{0.11} & \beststd{12.44}{0.39}  & \beststd{97.23}{0.08} \\
\rowcolor{gray!15}
SBOG (robust) & \secondstd{5.49}{0.57}  & \beststd{98.69}{0.09} & \std{50.60}{1.51} & \std{88.28}{0.39} & \beststd{2.96}{0.43}  & \secondstd{99.10}{0.07} & \secondstd{0.35}{0.13} & \secondstd{99.88}{0.02} & \beststd{5.46}{0.54}  & \beststd{98.78}{0.10} & \secondstd{12.97}{0.36}  & \secondstd{96.95}{0.05}\\
\hline
\end{tabular}
}\vspace{-8pt}
\end{table}

\begin{wraptable}{r}{0.4\textwidth}
\centering\vspace{-10pt}
\caption{Anchor-selection ablation.}\vspace{-6pt}
\label{tab:anchor_ablation}
\small
\setlength{\tabcolsep}{3.5pt}
\begin{tabular}{cccc}
\toprule
Strategy &  Token cos. & Diversity & Yield \\
\midrule
$k$-NN  & 0.490 & 0.635 & 0.277 \\
KDE-RBF  & 0.464 & 0.626 & 0.272 \\
SBOG  & \textbf{0.691}& \textbf{0.835} & \textbf{0.490} \\
\bottomrule
\end{tabular}
\vspace{-12pt}
\end{wraptable} 
We first test whether the boundary component of SBOG can be replaced by generic support scores.  We replace the SOE anchor rule with two statistical alternatives, $k$-NN sparsity and KDE-RBF scoring, while keeping the remaining proposal, semantic filtering, and decoding steps fixed.  Table~\ref{tab:anchor_ablation} shows that SOE gives better token alignment, diversity, and feasible yield. This indicates that the transport boundary energy selects more useful anchors than purely statistical support scores, even when the rest of the SBOG pipeline is unchanged.

\begin{figure}[h]\vspace{-10pt}
    \centering
    \includegraphics[width=0.92\linewidth]{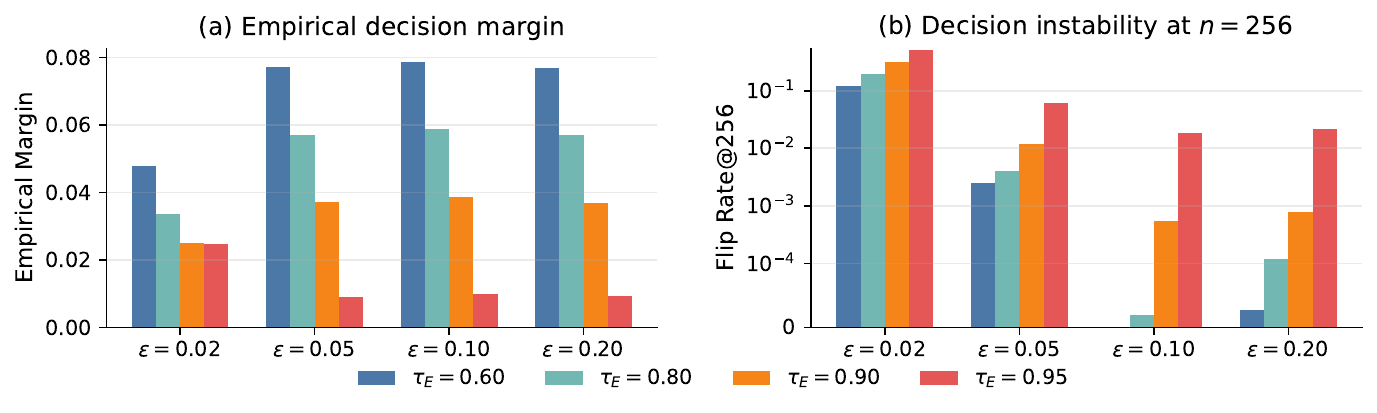}
    \caption{Finite-sample stability of the outlier-energy decision under varying $\epsilon$ and $m$.}
    \label{fig:sensitivity}
\end{figure}

\begin{wraptable}{r}{0.34\linewidth}
\vspace{-10pt}
\centering
\caption{Effect of semantic anchoring.}
\label{tab:candidate_quality}
\small
\setlength{\tabcolsep}{3pt}
\begin{tabular}{ccc}
\toprule
Metric & w/o filter & SBOG \\
\midrule
Purity $\uparrow$        & 0.381  & 0.824 \\
Token cos. $\uparrow$    & 0.329  & 0.602 \\
Class margin $\uparrow$  & -0.074 & 0.220 \\
\bottomrule
\end{tabular}
\vspace{-10pt}
\end{wraptable}
We next study the calibration of the SOE acceptance contour.  Using the same trained encoder, we sweep the threshold $\tau_E$ and report the empirical margin together with flip rate. Figure~\ref{fig:sensitivity} shows that very small $\varepsilon$ makes the energy landscape sharp and unstable, while moderate smoothing yields larger margins and lower flip rates. This supports the use of SOE as a smooth boundary criterion inside SBOG and is consistent with the finite-sample stability condition in Theorem~\ref{thm_soe}.

\begin{wrapfigure}{r}{0.46\linewidth}
    \centering
    \includegraphics[width=\linewidth]{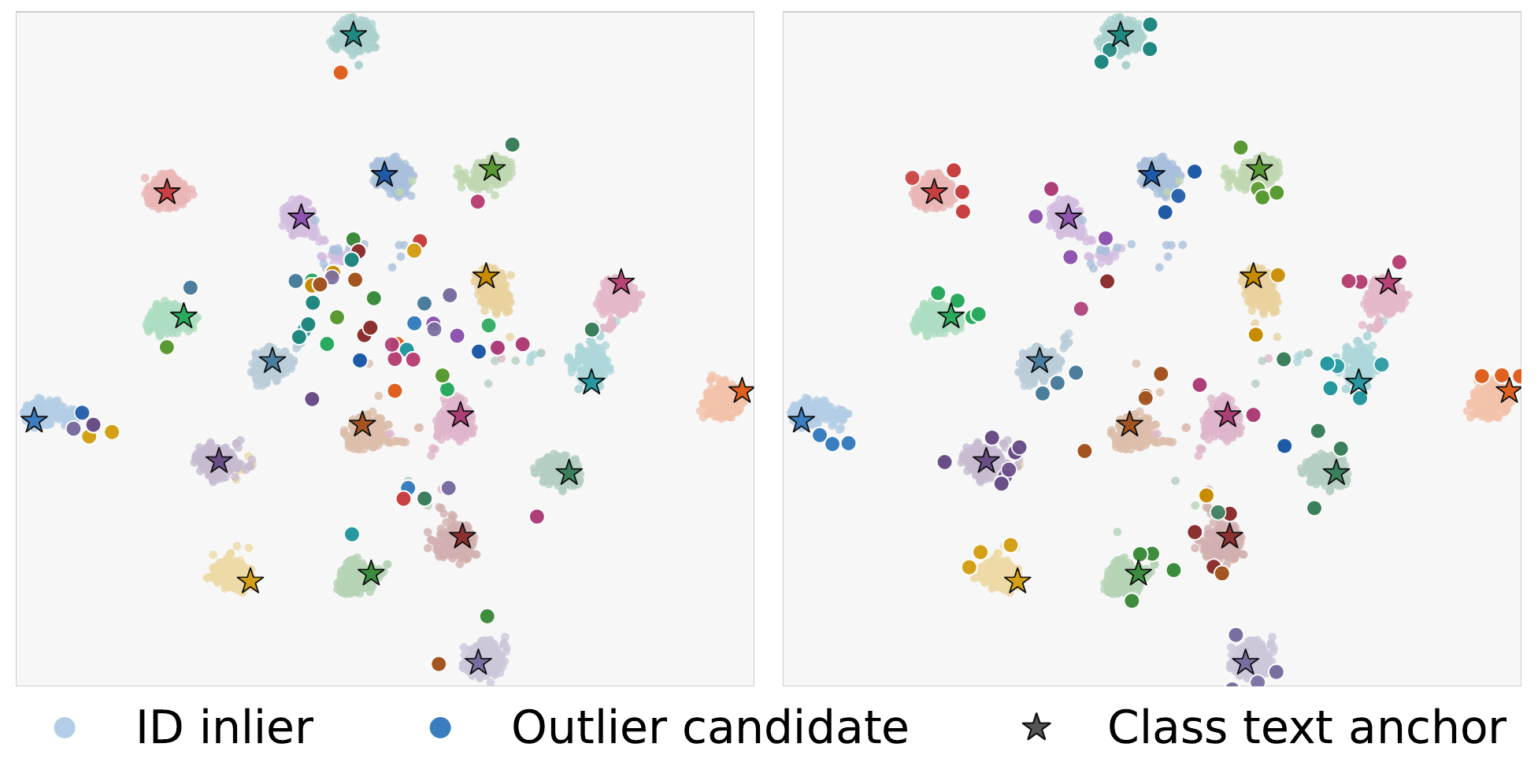}
    \caption{Visualization of anchors.}
    \label{fig:placeholder}
    \vspace{-12pt}
\end{wrapfigure} 

Finally, we isolate the role of semantic anchoring. 
The previous ablation shows that SOE provides stronger boundary anchors than $k$-NN or KDE, but weak support alone does not guarantee class consistency. 
We then compare an unanchored candidate pool with the token-anchored SBOG pool, using the same downstream ranking rule and evaluating the top $250$ candidates. Table~\ref{tab:candidate_quality} shows that semantic anchoring improves intended-class purity, token alignment, cross-class margin, and contamination control.  The t-SNE \cite{JMLR:v9:vandermaaten08a} in Figure~\ref{fig:placeholder} visualizes the same effect: anchored SBOG candidates remain near intended class boundaries, whereas unanchored candidates often mix with other class clusters or are absorbed into them.
These ablations separate the two roles in SBOG: SOE improves boundary selection over purely statistical support scores, while semantic anchoring prevents the selected weak-support candidates from drifting off class.

\section{Conclusion}
In this paper, we proposed Sinkhorn Boundary Outlier Generation (SBOG), a structured framework for latent-space outlier generation. SBOG uses Sinkhorn Outlier Energy (SOE) as a transport-support energy to identify weakly supported boundary regions, and combines it with semantic feasibility and decoding constraints to generate controlled, class-consistent outliers. We showed that SOE recovers kernel support as an empirical Euclidean special case, while its appearance in the Sinkhorn-DRO dual gives the boundary rule an optimization-consistent interpretation. Experiments on time-series anomaly generation and image outlier synthesis demonstrate that SBOG can generate informative synthetic outliers across modalities and improve downstream robustness evaluation.

SBOG depends on representation quality, transport-cost design, semantic anchoring, and generator expressiveness. Weak embeddings or limited decoders may reduce outlier fidelity and utility. Future work can develop adaptive calibration strategies for the Sinkhorn temperature and SOE threshold, jointly learn the latent representation, transport cost, and outlier generator, and extend the Sinkhorn-DRO interpretation toward downstream-aware stress generation.

\bibliography{ref}
\bibliographystyle{plain}
\clearpage
\appendix

\section{Related works} \label{supp_relatedworks}
\paragraph{DRO for representation} Classical DRO and its Wasserstein variants learn models against worst-case shifts through ambiguity sets over distributions \cite{shafieezadeh2019regularization,duchi2021learning,sinha2018certifying}, later extended to local perturbations and entropic Sinkhorn formulations \cite{wang2025sinkhorn}. A closer line uses DRO to shape robust representations under subgroup or covariate shift, for example by learning hierarchically robust features \cite{jeong2025multiexpert}, analyzing robust versus spurious representations \cite{kim2025sufficient,sun2023alternating}, or coupling invariant representation learning with worst-group objectives \cite{qian2020hierarchically,zhou2021examining,sunrobust}. More recent work connects DRO to self-supervised and latent representation learning, but they are mostly concerned about interpretation \cite{wu2023understanding,sun2026diff2sp} or designing loss function for robustness in worst-case \cite{ghosh2024distributionally}. However, even in those work consider the outlier or negative points explicitly \cite{hu2018does,Sagawa2020Distributionally,izmailov2022feature,pmlr-v151-slowik22a}, DRO is still used mainly to reweight losses or stabilize learned features thereby stabilize features and working on loss function refinement. Instead, our framework turns Sinkhorn-DRO geometry into an explicit and rigorous definition of outlyingness in latent space, yielding a training-consistent criterion for outlier identification and synthesis that is naturally more general than modality-specific generation pipelines.

\paragraph{Outlier in time series}
Most existing work on time-series outliers is still detection-oriented \cite{NEURIPS2020_97e401a0,liu2024the,zamanzadeh2024deep}, focusing on learning a reliable notion of normality rather than explicitly modeling how outliers are formed. Typical approaches either learn compact normal regions, sometimes with weak anomaly supervision to calibrate the boundary \cite{pmlr-v80-ruff18a,pang2019deep, Ruff2020Deep,xu2024calibrated}, or strengthen detector backbones through reconstruction-based scoring \cite{usad,garg2021evaluation,tuli2022tranad}, contextual comparison \cite{ijcai2022p394}, and representation learning with attention or contrastive objectives \cite{xu2022anomaly,yang2023dcdetector}. Recent work has mostly strengthened detector backbones instead of rethinking outlier formation, by improving temporal structure modeling \cite{wu2022timesnet,licrossad}, memory modules \cite{song2023memto}, hybrid prototypes \cite{shen2025learn}, and exploiting masking or frequency-domain learning \cite{10597757,wucatch}. This trend continues in newer interpretable and foundation-style detectors \cite{chostructured,lan2025towards,yinscatterad,martinez2026adaptive,zhang2026icdiffad} and. By contrast, work that explicitly generates time-series outliers remains scarce. Most anomaly simulation is only auxiliary to detector training \cite{darban2025carla,xu2024calibrated}, while dedicated synthesis methods are still rare \cite{darban2025genias,app14177714}. Relative to these baselines, our focus is not merely to separate normal from abnormal windows, but to define and synthesize informative outliers from a principled latent geometry, making the generated samples directly useful for further downstream evaluation and analysis.

\paragraph{Image outlier synthesis}
For image data, the literature on synthetic outliers has mostly been driven by OOD detection rather than outlier generation itself. Early training-based methods improved detectors by using curated auxiliary negatives to regularize classifiers rather than explicitly generating new outliers \cite{hendrycks2018deep,liu2020energy}. Later work began to synthesize outliers in feature space \cite{lee2018training}, either with learned generators or with virtual sampling schemes \cite{du2022towards,npos,li2025outlier}, which improves boundary regularization but still defines candidate outliers through empirical density or sampling heuristics. More recent methods further push synthesis to the image level: \cite{du2023dream} decodes low-likelihood latent samples with diffusion, while \cite{yoon2025diffusion,chen2024fodfom,liao2025bood,gao2025good} generate more realistic visual outliers through diffusion or foundation-model guidance. Despite the stronger visual quality, most of these methods still localize candidate outliers through heuristic proxies such as low-likelihood regions \cite{du2023dream}, sparse $k$-NN neighborhoods \cite{npos,gao2025good}, or boundary-crossing perturbations \cite{liao2025bood}, rather than from a principled definition of outlyingness for detection. In contrast, our method is more general and not restricted to certain modality. Besides, we provide a rigorous definition of outlyingness through Sinkhorn geometric boundary, so outlier selection and synthesis are both grounded in a principled geometry rather than heuristic and purely statistic rules.

\section{Discussion of Sinkhorn Boundary Outlier Energy}
We record the variational identity underlying the Sinkhorn outlier energy in
Definition~\ref{def:outlier_energy}. For any $v\in\mathcal Z$, recall that
\[
K^\nu_\varepsilon(v)
:=
\int_{\mathcal Z}
\exp\!\left(-\frac{d(v,u)}{\varepsilon}\right)d\nu(u),
\qquad
E^\nu_\varepsilon(v)
:=
-\varepsilon\log K^\nu_\varepsilon(v).
\]
Assume $0<K^\nu_\varepsilon(v)<\infty$. Then
\begin{equation}
\inf_{r \in \mathcal P(\mathcal Z)}
\left\{
\mathbb E_{u \sim r}[d(v,u)] + \varepsilon \mathrm{KL}(r\|\nu)
\right\}
=
-\varepsilon \log K^\nu_\varepsilon(v)
=
E^\nu_\varepsilon(v).
\label{eq:soe-variational}
\end{equation}
Therefore, $E^\nu_\varepsilon(v)$ is exactly the optimal value of the entropic
cost--KL tradeoff centered at $v$ relative to the reference measure $\nu$. Larger
values of $E^\nu_\varepsilon(v)$ correspond to smaller kernel mass under $\nu$,
meaning that $v$ is less supported by the reference geometry and is therefore more
likely to be an outlier under the proposed energy criterion.

Fix $v \in \mathcal Z$, and define the Gibbs-tilted probability measure
$r_{v,\varepsilon}^{\nu,\star}\in\mathcal P(\mathcal Z)$ by
\begin{equation*}
dr_{v,\varepsilon}^{\nu,\star}(u)
:=
\frac{\exp(-d(v,u)/\varepsilon)}{K^\nu_\varepsilon(v)}d\nu(u).
\end{equation*}
Then for any $r \ll \nu$,
\begin{align*}
\mathrm{KL}(r\|r_{v,\varepsilon}^{\nu,\star})
&=
\int_{\mathcal Z}
\log\!\left(\frac{dr}{dr_{v,\varepsilon}^{\nu,\star}}\right)dr \\
&=
\int_{\mathcal Z}
\log\!\left(\frac{dr}{d\nu}\right)dr
+
\int_{\mathcal Z}
\frac{d(v,u)}{\varepsilon}dr(u)
+
\log K^\nu_\varepsilon(v) \\
&=
\mathrm{KL}(r\|\nu)
+
\frac{1}{\varepsilon}\mathbb E_{u\sim r}[d(v,u)]
+
\log K^\nu_\varepsilon(v).
\end{align*}
Rearranging gives
\begin{equation*}
\mathbb E_{u\sim r}[d(v,u)] + \varepsilon\mathrm{KL}(r\|\nu)
=
-\varepsilon \log K^\nu_\varepsilon(v)
+
\varepsilon\mathrm{KL}(r\|r_{v,\varepsilon}^{\nu,\star})
\ge
-\varepsilon \log K^\nu_\varepsilon(v),
\end{equation*}
with equality if and only if $r = r_{v,\varepsilon}^{\nu,\star}$. This proves
\eqref{eq:soe-variational}. Hence the definition
\[
E^\nu_\varepsilon(v)
=
-\varepsilon \log K^\nu_\varepsilon(v)
\]
is exactly the entropic variational cost associated with $v$ relative to the
reference measure $\nu$.

Finally, the thresholded outlier set
\[
O_{\nu,\varepsilon}(\tau_E)
:=
\{v \in \mathcal Z : E^\nu_\varepsilon(v) \ge \tau_E\}
\]
is equivalently characterized by
\[
K^\nu_\varepsilon(v) \le e^{-\tau_E/\varepsilon}.
\]

Similarly, we give the proof of Lemma \ref{lem:soe_support_cost}.

\textbf{Lemma 3.2} (Entropic support cost of Sinkhorn outlier energy)
\textit{For any $z\in Z$ such that $0<K^\nu_\epsilon(z)<\infty$, given the Sinkhorn outlier
energy $E^\nu_\epsilon(z)$, in particular, for the empirical reference measure
$\hat\nu_n=\frac1n\sum_{i=1}^n\delta_{z_i}$ with distinct atoms, we have
\[
E^{\hat\nu_n}_\epsilon(z)
=
\min_{w\in\Delta_n}
\left\{
\sum_{i=1}^n w_i d(z,z_i)
+
\epsilon\sum_{i=1}^n w_i\log(nw_i)
\right\},
\]
where $\Delta_n=\{w\in\mathbb R_+^n:\sum_{i=1}^n w_i=1\}$ with the minimizer being $w_i^\star(z) =\frac{\exp\!\left(-d(z,z_i)/\epsilon\right)}
{\sum_{j=1}^n \exp\!\left(-d(z,z_j)/\epsilon\right)}$.
}
\begin{proof}
Fix $z\in Z$ with $0<K^\nu_\epsilon(z)<\infty$, and define
\[
\frac{dr_z^\star}{d\nu}(u)
=
\frac{\exp\!\left(-d(z,u)/\epsilon\right)}
{K^\nu_\epsilon(z)} .
\]
This is a probability density with respect to $\nu$. For any
$r\in\mathcal P(Z)$ with $r\ll\nu$, using the chain rule for Radon--Nikodym
derivatives gives
\[
\log\frac{dr}{dr_z^\star}(u)
=
\log\frac{dr}{d\nu}(u)
+
\frac{d(z,u)}{\epsilon}
+
\log K^\nu_\epsilon(z).
\]
Therefore, in the extended-real sense,
\[
\epsilon\mathrm{KL}(r\|r_z^\star)
=
\epsilon\mathrm{KL}(r\|\nu)
+
\mathbb E_{u\sim r}[d(z,u)]
+
\epsilon\log K^\nu_\epsilon(z).
\]
Rearranging yields
\[
\mathbb E_{u\sim r}[d(z,u)]
+
\epsilon\mathrm{KL}(r\|\nu)
=
-\epsilon\log K^\nu_\epsilon(z)
+
\epsilon\mathrm{KL}(r\|r_z^\star).
\]
Since $\mathrm{KL}(r\|r_z^\star)\ge 0$, the right-hand side is minimized by
$r=r_z^\star$, and the minimum value is
\[
-\epsilon\log K^\nu_\epsilon(z)
=
E^\nu_\epsilon(z).
\]

For $\hat\nu_n=\frac1n\sum_{i=1}^n\delta_{z_i}$ with distinct atoms, any
$r\ll\hat\nu_n$ can be written as $r=\sum_{i=1}^n w_i\delta_{z_i}$ for some
$w\in\Delta_n$. In this case,
\[
\mathrm{KL}(r\|\hat\nu_n)
=
\sum_{i=1}^n w_i\log(nw_i),
\]
and substituting this expression into the variational formula gives the empirical
form. The expression for $w_i^\star(z)$ follows by substituting
$\nu=\hat\nu_n$ into the Gibbs reweighting above.
\end{proof}

\textbf{Proposition 3.3} (Outlier energy appears in the Sinkhorn-DRO dual)
\textit{Following the previous definitions, for any measurable $f:\mathcal Z\to\mathbb R$ such that the log-moment below is finite for $\lambda>0$, the Sinkhorn-DRO dual objective admits the following
decomposition: for every $\lambda>0$,
\begin{align*}
V_D(\lambda)
&=
\lambda \tau_E
+
\lambda\varepsilon
\mathbb E_{z\sim\mu}
\left[
\log
\mathbb E_{u\sim Q_{\varepsilon,z}^{\nu}}
\exp\!\left(
\frac{f(u)}{\lambda\varepsilon}
\right)
\right]
\nonumber\\
&=
\lambda \tau_E
+
\mathbb E_{z\sim\mu}
\left[
\lambda\varepsilon
\log
\int_{\mathcal Z}
\exp\!\left(
\frac{f(u)}{\lambda\varepsilon}
-
\frac{d(z,u)}{\varepsilon}
\right)
d\nu(u)
+
\lambda E^\nu_\varepsilon(z)
\right],
\label{eq:VD-decomp}
\end{align*}
where $Q_{\varepsilon,z}^{\nu}$ is the Gibbs kernel and $E^\nu_\varepsilon(z)= -\varepsilon\log K^\nu_\varepsilon(z)$ is the Sinkhorn outlier energy from Definition~\ref{def:outlier_energy}. In particular, minimizing the Sinkhorn-DRO objective necessarily optimizes an objective that contains the outlier energy term $\lambda\mathbb E_{z\sim\mu}\!\left[E^\nu_\varepsilon(z)\right]$.}

\begin{proof}
Fix $\lambda>0$. By definition of the Gibbs kernel from Lemma~\ref{lem:sinkhorndro},
\[
dQ_{\varepsilon,v}^{\nu}(u)
=
\frac{\exp\!\big(-d(z,u)/\varepsilon\big)}
{K^\nu_\varepsilon(z)}
d\nu(u),
\qquad
K^\nu_\varepsilon(z)
=
\int_{\mathcal Z}
\exp\!\big(-d(z,u)/\varepsilon\big)d\nu(u).
\]
For each fixed representation $z\in\mathcal Z$, expand the inner expectation:
\begin{align*}
\mathbb E_{u\sim Q_{\varepsilon,z}^{\nu}}
\exp\!\left(\frac{f(u)}{\lambda\varepsilon}\right)
&=
\int_{\mathcal Z}
\exp\!\left(\frac{f(u)}{\lambda\varepsilon}\right)
dQ_{\varepsilon,z}^{\nu}(u)\\
&=
\int_{\mathcal Z}
\exp\!\left(\frac{f(u)}{\lambda\varepsilon}\right)
\frac{\exp\!\big(-d(z,u)/\varepsilon\big)}
{K^\nu_\varepsilon(z)}
d\nu(u)\\
&=
\frac{1}{K^\nu_\varepsilon(z)}
\int_{\mathcal Z}
\exp\!\left(
\frac{f(u)}{\lambda\varepsilon}
-
\frac{d(z,u)}{\varepsilon}
\right)d\nu(u).
\end{align*}
Taking $\log$ gives
\[
\log
\mathbb E_{u\sim Q_{\varepsilon,z}^{\nu}}
\exp\!\left(\frac{f(u)}{\lambda\varepsilon}\right)
=
\log\!\int_{\mathcal Z}
\exp\!\left(
\frac{f(u)}{\lambda\varepsilon}
-
\frac{d(z,u)}{\varepsilon}
\right)d\nu(u)
-
\log K^\nu_\varepsilon(z).
\]
Multiplying by $\lambda\varepsilon$ yields
\[
\lambda\varepsilon
\log
\mathbb E_{u\sim Q_{\varepsilon,z}^{\nu}}
\exp\!\left(\frac{f(u)}{\lambda\varepsilon}\right)
=
\lambda\varepsilon
\log\!\int_{\mathcal Z}
\exp\!\left(
\frac{f(u)}{\lambda\varepsilon}
-
\frac{d(z,u)}{\varepsilon}
\right)d\nu(u)
-
\lambda\varepsilon\log K^\nu_\varepsilon(z).
\]
By Definition~\ref{def:outlier_energy}, $E^\nu_\varepsilon(z)=-\varepsilon\log K^\nu_\varepsilon(z)$, and hence
\[
-\lambda\varepsilon\log K^\nu_\varepsilon(z)
=
\lambda E^\nu_\varepsilon(z).
\]
Substituting this identity into the Sinkhorn-DRO dual objective from 
Lemma~\ref{lem:sinkhorndro} and taking $\mathbb E_{v\sim\mu}[\cdot]$ proves the 
decomposition \eqref{eq:VD-decomp}. The final statement follows because the 
normalization term defining $E^\nu_\varepsilon(v)$ appears explicitly in this 
decomposition of the robust log-moment term.
\end{proof}

\section{Comparison of Sample Complexity}
We give discussion about sample complexity about neighborbood selection and our Sinkhorn outlier energy. To avoid confounding the comparison with pipeline differences, we keep all components fixed, and vary only the boundary scoring rule used in the sampling step.  We first take $k$-NN as an example, then formalize the sample complexity by a hard neighborhood-based rule.
\label{app_thm_neighbor}

\begin{theorem}[Sample Complexity of Hard Neighborhood Selection]
\label{thm:knn_sample_complexity}
For any integer $k\ge 1$ and any $v\in\mathcal Z$, define the empirical $k$-NN radius by 
\begin{equation}
\widehat r_{n,k}(v) := \inf\left\{r>0 : \widehat{\nu}_n(B(v,r)) \ge \frac{k}{n}\right\},
\qquad
B(v,r) := \{u \in \mathcal Z : d(v,u) \le r\}.
\end{equation} Assume that there exist constants $C>0$, $m>0$, and $r_0>0$ such that $\bar{\nu}(B(v,r))\le Cr^m$ for all $v\in\mathcal V$ and $ r\in(0,r_0]$, where $\mathcal V\subseteq\mathcal Z$ is the set of candidate query points. Then for any fixed $v\in\mathcal V$ and any $0<\Delta_r\le r_0$, if $n < \frac{k}{2C\Delta_r^m}$, then $\mathbb P\!\left(\widehat r_{n,k}(v)>\Delta_r\right)\ge1-e^{-k/6}$.
\end{theorem}

\begin{proof}
Fix $v\in\mathcal V$ and define $N_n(v,r)
:=\sum_{i=1}^n \mathbf 1\{z_i\in B(v,r)\}$. Then $N_n(v,r)$ is a binomial random variable with success probability $p_v(r):=\bar{\nu}(B(v,r))$. By assumption, for every $r\in(0,r_0]$, $p_v(r)\le Cr^m$. Fix $0<\Delta_r\le r_0$. If $n < \frac{k}{2C\Delta_r^m}$, then
\begin{equation}
\mu_v
:=
\mathbb E[N_n(v,\Delta_r)]
=
np_v(\Delta_r)
\le
nC\Delta_r^m
<
\frac{k}{2}.
\end{equation}
By the definition of $\widehat r_{n,k}(v)$, $\{\widehat r_{n,k}(v)\le \Delta_r\}=\{N_n(v,\Delta_r)\ge k\}$.
Since $\mu_v<k/2$, we have $k\ge 2\mu_v$. Applying the multiplicative Chernoff
bound for the upper tail of a binomial random variable gives
\begin{equation}
\mathbb P\!\left(N_n(v,\Delta_r)\ge k\right)
\le
\exp\!\left(-\frac{k-\mu_v}{3}\right)
\le
e^{-k/6}.
\end{equation}
Therefore,
\[
\mathbb P\!\left(\widehat r_{n,k}(v)\le \Delta_r\right)
\le
e^{-k/6},
\]
or equivalently,
\[
\mathbb P\!\left(\widehat r_{n,k}(v)>\Delta_r\right)
\ge
1-e^{-k/6}.
\]
This proves the claim for the fixed query point $v\in\mathcal V$.
\end{proof}

Theorem~\ref{thm:knn_sample_complexity} shows that, under the local mass upper bound
$\bar{\nu}(B(v,r)) \le C r^m$, a hard $k$-NN rule cannot reliably resolve a neighborhood scale
$\Delta_r$ unless the sample size is on the order of $k\Delta_r^{-m}$. This reflects the cost of
estimating a local counting or radius event in weakly supported regions. The SOE rule is based on a
different statistical object: for a fixed proposal pool, it asks whether the smooth Gibbs support cost
crosses a threshold $\tau_E$. The next theorem gives a finite-candidate stability guarantee for this
threshold decision under explicit normalizer and margin conditions.

We next prove Theorem \ref{thm_soe}.

\textbf{Theorem 3.4} (Finite-sample stability of SOE thresholding) \textit{For a fixed finite candidate pool $\mathcal V \subset \mathcal Z$, define $\underline K:=\inf_{v \in \mathcal V}K^{\bar{\nu}}_\varepsilon(v) > 0$, $\Delta_E=\inf_{v \in \mathcal V}\left|E^{\bar{\nu}}_\varepsilon(v) - \tau_E\right| > 0$, and $\tau_n(\delta)=\sqrt{{\log(2|\mathcal V|/\delta)}/{2n}}$. If $\tau_n(\delta) \le \underline K/2$, then with probability at least $1-\delta$,
\[
\sup_{v \in \mathcal V}
\left|
E^{\widehat{\nu}_n}_\varepsilon(v)
-
E^{\bar{\nu}}_\varepsilon(v)
\right|
\le
\frac{2\varepsilon}{\underline K}\tau_n(\delta).
\]
Consequently, if $\tau_n(\delta)\le\min\!\left\{\frac{\underline K}{2},\frac{\underline K \Delta_E}{4\varepsilon}\right\}$, then with probability at least $1-\delta$, $\mathbf 1\{E^{\widehat{\nu}_n}_\varepsilon(v) \ge \tau_E\}=\mathbf 1\ E^{\bar{\nu}}_\varepsilon(v) \ge \tau_E\}$ for any $v \in \mathcal V$. In particular, it suffices that
\[
n \ge
\max\!\left\{
\frac{2}{\underline K^2}\log\frac{2|\mathcal V|}{\delta},
\;
\frac{8\varepsilon^2}{\underline K^2\Delta_E^2}
\log\frac{2|\mathcal V|}{\delta}
\right\}.
\]}
\begin{proof}
Fix a finite candidate pool $\mathcal V \subset \mathcal Z$, and define $\psi_v(u):=\exp\!\left(-\frac{d(v,u)}{\varepsilon}\right)$. Since $d(v,u) \ge 0$, we have $0 < \psi_v(u) \le 1$ for all
$u \in \mathcal Z$. By the definition of $K^\nu_\varepsilon$,
\begin{equation*}
K^{\bar{\nu}}_\varepsilon(v)
=
\int_{\mathcal Z}
\exp\!\left(-\frac{d(v,u)}{\varepsilon}\right)d\bar{\nu}(u),
\qquad
K^{\widehat{\nu}_n}_\varepsilon(v)
=
\frac{1}{n}\sum_{i=1}^n
\exp\!\left(-\frac{d(v,z_i)}{\varepsilon}\right).
\end{equation*}
Therefore, for each fixed $v \in \mathcal V$, the quantity
$K^{\widehat{\nu}_n}_\varepsilon(v)$ is the empirical mean of $n$ i.i.d.
random variables in $[0,1]$ with mean $K^{\bar{\nu}}_\varepsilon(v)$.

By Hoeffding's inequality, for any $t>0$ and any fixed $v \in \mathcal V$,
\begin{equation*}
\mathbb P\!\left(
\left|
K^{\widehat{\nu}_n}_\varepsilon(v)
-
K^{\bar{\nu}}_\varepsilon(v)
\right| > t
\right)
\le
2e^{-2nt^2}.
\end{equation*}
Applying the union bound over $v \in \mathcal V$ gives
\begin{equation*}
\mathbb P\!\left(
\sup_{v \in \mathcal V}
\left|
K^{\widehat{\nu}_n}_\varepsilon(v)
-
K^{\bar{\nu}}_\varepsilon(v)
\right| > t
\right)
\le
2|\mathcal V|e^{-2nt^2}.
\end{equation*}
Choosing $t = \tau_n(\delta)$ yields
\begin{equation*}
\mathbb P\!\left(
\sup_{v \in \mathcal V}
\left|
K^{\widehat{\nu}_n}_\varepsilon(v)
-
K^{\bar{\nu}}_\varepsilon(v)
\right|
\le
\tau_n(\delta)
\right)
\ge
1-\delta.
\end{equation*}
Hence, with probability at least $1-\delta$, the event $\mathcal E_n:=\left\{\sup_{v \in \mathcal V}\left|K^{\widehat{\nu}_n}_\varepsilon(v)-K^{\bar{\nu}}_\varepsilon(v)\right|\le\tau_n(\delta)\right\}$
holds. We now work on $\mathcal E_n$. Since $\underline K:=\inf_{v \in \mathcal V} K^{\bar{\nu}}_\varepsilon(v) > 0$, the condition $\tau_n(\delta) \le \underline K/2$ implies
\begin{equation*}
K^{\widehat{\nu}_n}_\varepsilon(v)
\ge
K^{\bar{\nu}}_\varepsilon(v) - \tau_n(\delta)
\ge
\underline K - \tau_n(\delta)
\ge
\frac{\underline K}{2},
\qquad \forall v \in \mathcal V.
\end{equation*}
Also, since $0 < \psi_v(u) \le 1$, we have
$K^{\bar{\nu}}_\varepsilon(v) \le 1$ and
$K^{\widehat{\nu}_n}_\varepsilon(v) \le 1$. Thus both
$K^{\bar{\nu}}_\varepsilon(v)$ and $K^{\widehat{\nu}_n}_\varepsilon(v)$
lie in $[\underline K/2,1]$ for every $v \in \mathcal V$.

Consider the function $\phi(x):=-\varepsilon\log x$ on $(0,\infty)$.
Its derivative satisfies $\phi'(x)=-\varepsilon/x$, so on
$[\underline K/2,1]$ we have
\[
|\phi'(x)| \le \frac{2\varepsilon}{\underline K}.
\]
Hence $\phi$ is $(2\varepsilon/\underline K)$-Lipschitz on
$[\underline K/2,1]$. Using the definition of $E^\nu_\varepsilon$, for every
$v \in \mathcal V$,
\begin{equation*}
\left|
E^{\widehat{\nu}_n}_\varepsilon(v)
-
E^{\bar{\nu}}_\varepsilon(v)
\right|
=
\left|
\phi(K^{\widehat{\nu}_n}_\varepsilon(v))
-
\phi(K^{\bar{\nu}}_\varepsilon(v))
\right|
\le
\frac{2\varepsilon}{\underline K}
\left|
K^{\widehat{\nu}_n}_\varepsilon(v)
-
K^{\bar{\nu}}_\varepsilon(v)
\right|.
\end{equation*}
Taking the supremum over $\mathcal V$ and using $\mathcal E_n$, we obtain
\begin{equation*}
\sup_{v \in \mathcal V}
\left|
E^{\widehat{\nu}_n}_\varepsilon(v)
-
E^{\bar{\nu}}_\varepsilon(v)
\right|
\le
\frac{2\varepsilon}{\underline K}\tau_n(\delta),
\end{equation*}
which proves the first claim.

For the second claim, assume in addition that $\frac{2\varepsilon}{\underline K}\tau_n(\delta)\le\frac{\Delta_E}{2}$ for $\Delta_E:=\inf_{v \in \mathcal V}\left|E^{\bar{\nu}}_\varepsilon(v)-\tau_E \right| > 0$. Fix any $v \in \mathcal V$, if
$E^{\bar{\nu}}_\varepsilon(v) \ge \tau_E$, then by the definition of
$\Delta_E$, we have
\[
E^{\bar{\nu}}_\varepsilon(v)-\tau_E
=
\left|
E^{\bar{\nu}}_\varepsilon(v)-\tau_E
\right|
\ge
\Delta_E,
\]
hence $E^{\bar{\nu}}_\varepsilon(v) \ge \tau_E+\Delta_E$. Using the uniform energy bound proved above,
\begin{equation*}
E^{\widehat{\nu}_n}_\varepsilon(v)\ge E^{\bar{\nu}}_\varepsilon(v)-\left|E^{\widehat{\nu}_n}_\varepsilon(v)-E^{\bar{\nu}}_\varepsilon(v)\right|\ge(\tau_E+\Delta_E)-\frac{\Delta_E}{2}=\tau_E+\frac{\Delta_E}{2}>\tau_E.
\end{equation*}
Therefore, $\mathbf 1\{E^{\widehat{\nu}_n}_\varepsilon(v) \ge \tau_E\}=1$.

If instead $E^{\bar{\nu}}_\varepsilon(v) < \tau_E$, then $\tau_E-E^{\bar{\nu}}_\varepsilon(v)=\left|E^{\bar{\nu}}_\varepsilon(v)-\tau_E\right|\ge\Delta_E$, so $E^{\bar{\nu}}_\varepsilon(v) \le \tau_E-\Delta_E$. Again using the uniform energy bound,
\begin{equation*}
E^{\widehat{\nu}_n}_\varepsilon(v)
\le
E^{\bar{\nu}}_\varepsilon(v)
+
\left|
E^{\widehat{\nu}_n}_\varepsilon(v)
-
E^{\bar{\nu}}_\varepsilon(v)
\right|
\le
(\tau_E-\Delta_E)+\frac{\Delta_E}{2}
=
\tau_E-\frac{\Delta_E}{2}
<
\tau_E.
\end{equation*}
Therefore,
\[
\mathbf 1\{E^{\widehat{\nu}_n}_\varepsilon(v) \ge \tau_E\}=0.
\]

Thus, on the event $\mathcal E_n$,
\begin{equation*}
\mathbf 1\{E^{\widehat{\nu}_n}_\varepsilon(v) \ge \tau_E\}
=
\mathbf 1\{E^{\bar{\nu}}_\varepsilon(v) \ge \tau_E\},
\qquad \forall v \in \mathcal V.
\end{equation*}
Since $\mathbb P(\mathcal E_n)\ge 1-\delta$, the conclusion follows with
probability at least $1-\delta$.

Finally, the conditions $\tau_n(\delta) \le \frac{\underline K}{2}$ and $\frac{2\varepsilon}{\underline K}\tau_n(\delta)\le\frac{\Delta_E}{2}$ are implied by
\begin{equation*}
\tau_n(\delta)
\le
\min\!\left\{
\frac{\underline K}{2},
\frac{\underline K\Delta_E}{4\varepsilon}
\right\}.
\end{equation*}
Rearranging these two inequalities yields the sufficient sample-size bound
\begin{equation*}
n \ge
\max\!\left\{
\frac{2}{\underline K^2}\log\frac{2|\mathcal V|}{\delta},
\;
\frac{8\varepsilon^2}{\underline K^2\Delta_E^2}
\log\frac{2|\mathcal V|}{\delta}
\right\}.
\end{equation*}
This completes the proof.
\end{proof}

\section{Population-SOE consistency of the sampler}
\label{app:sbog_population_consistency}

We give the definitions and proof for Theorem~\ref{thm:sbog_population_consistency}. Unless stated otherwise, probabilities in this subsection are conditional on the observed bank and the fixed representation. The empirical and population-SOE samplers both use Algorithm \ref{alg:outlier_sampling}, including redrawing the source anchor after an unsuccessful round. The population-SOE sampler replaces the energy evaluations, not the observed anchor candidates, by their population counterparts.

\paragraph{Threshold calibration.}
Let $s_i=E(z_i)$ and $\widehat s_i=\widehat E_n(z_i)$. Write the thresholds as $\tau_n^\star=\mathcal T_n(s_1,\ldots,s_n)$ and $\widehat\tau_n=\mathcal T_n(\widehat s_1,\ldots,\widehat s_n)$, where the common calibration map satisfies
\begin{equation}
\label{eq:sbog_threshold_lipschitz}
|\mathcal T_n(u)-\mathcal T_n(v)|
\le \|u-v\|_\infty.
\end{equation}
A common fixed threshold satisfies this condition. So does a fixed empirical quantile, computed by an order statistic or a fixed interpolation of adjacent order statistics, plus a common additive margin. Indeed, perturbing every score by at most $t$ perturbs every order statistic, and hence such a quantile, by at most $t$.

\paragraph{Anchor-selection margin.}
Let $1\le A=A_n\le n$. Let $I_A^\star$ and $\widehat I_{A,n}$ contain
the indices of the top-$A$ population and empirical bank energies,
respectively, using the same deterministic tie-breaking rule.
For $A<n$, sort the population energies as
$s_{(1)}\ge\cdots\ge s_{(n)}$ and define
\[
\xi_{A,n}=\frac{s_{(A)}+s_{(A+1)}}{2},
\qquad
a_n(t)=\min\!\left\{1,
\frac{1}{2A}\sum_{i=1}^n
\mathbf 1\{|s_i-\xi_{A,n}|\le2t\}\right\}.
\]
For $A=n$, set $a_n(t)=0$.
This quantity bounds the fraction of the uniform anchor law that can
change under an energy perturbation of size $t$.

\paragraph{Proposal-selection margin.}
A population proposal round first draws
$I\sim\operatorname{Unif}(I_A^\star)$ and then draws
\[
V_j=\Pi_{\mathcal Z}(z_I+\eta_j),
\qquad
\eta_j\overset{\mathrm{i.i.d.}}{\sim}
\mathcal N(0,\sigma^2 I_m),
\qquad j=1,\ldots,M.
\]
Let $R_n^\star$ denote the joint law of $(I,V_1,\ldots,V_M)$.
Define the semantic and population-feasible index sets by $S=\{j:\cos(V_j,t_{y_I})\ge\tau_{\mathrm{sem}}\}$ and $J^\star=\{j\in S:E(V_j)\ge\tau_n^\star\}$.s
Let $\Delta_n$ be the difference between the largest and
second-largest values of $E(V_j)$ over $j\in J^\star$.
Set $\Delta_n=+\infty$ when $|J^\star|\le1$, and use
$\min\varnothing=+\infty$. Define
\begin{align}
\label{eq:sbog_proposal_margin}
b_n(t)
&=R_n^\star\!\left(
 \left\{\min_{j\in S}|E(V_j)-\tau_n^\star|\le2t\right\}
 \cup\{\Delta_n\le2t\}\right),\\
\label{eq:sbog_round_success}
p_n^\star&=R_n^\star(J^\star\ne\varnothing).
\end{align}
The factor $2t$ in the threshold event accounts for errors in both
the candidate energy and the calibrated threshold. The second event
accounts for instability of the energy-maximizing feasible proposal.
No semantic-margin assumption is needed because the coupled samplers
use the same semantic test on the same proposals.

\textbf{Theorem 3.5} (Population-SOE consistency) \textit{Suppose $|\widehat\tau_n-\tau_n^\star|\le e_n$,
$p_n^\star\ge p_0>0$, and $a_n+b_n(e_n)<p_0$.
Then the empirical sampler terminates almost surely, and
\begin{equation*}
\label{eq:sbog_population_tv}
d_{\mathrm{TV}}\!\left(
Q_n^{\mathrm{SBOG}},Q_{n,\mathrm{SOE}}^\star
\right)
\le
\min\!\left\{1,\frac{2\{a_n+b_n(e_n)\}}{p_0}\right\}
=:\zeta_n,
\end{equation*}
where $d_{\mathrm{TV}}(P,Q)=\sup_B|P(B)-Q(B)|$. The same bound holds after applying a shared conditional decoder. If $e_n\to0$ and $a_n+b_n(e_n)\to0$ in probability over the bank, and the threshold and acceptance conditions hold with probability tending to one, then $d_{\mathrm{TV}}(Q_n^{\mathrm{SBOG}},Q_{n,\mathrm{SOE}}^\star)\to0$ in probability.}

\begin{proof}
Fix a bank satisfying the theorem's conditions and set
$\eta_n=a_n(t)+b_n(t)$.
We first bound disagreement in one proposal round and then account
for repeated rejection.

\emph{Anchor selection.}
Let $r=|\widehat I_{A,n}\setminus I_A^\star|
=|I_A^\star\setminus\widehat I_{A,n}|$.
The total variation distance between the two uniform anchor laws is
$r/A$. When $r>0$, take any
$j\in\widehat I_{A,n}\setminus I_A^\star$ and
$k\in I_A^\star\setminus\widehat I_{A,n}$.
The empirical ranking gives $\widehat s_j\ge\widehat s_k$.
Therefore,
\[
s_k-s_j
\le |s_k-\widehat s_k|+|\widehat s_j-s_j|
\le2t.
\]
Since $s_k\ge s_{(A)}\ge\xi_{A,n}$ and
$s_j\le s_{(A+1)}\le\xi_{A,n}$, both indices lie in the band
$|s_i-\xi_{A,n}|\le2t$.
All $2r$ differing indices lie in this band, so
\[
d_{\mathrm{TV}}\!\left(
\operatorname{Unif}(\widehat I_{A,n}),
\operatorname{Unif}(I_A^\star)\right)
=\frac rA\le a_n(t).
\]
The claim is immediate when $A=n$.
Couple the two anchor draws so that they differ with probability at
most $a_n(t)$. Such a coupling assigns the common probability mass
to identical indices and couples the remaining mass separately.

\emph{Feasibility and maximization.}
When the anchors agree, use the same Gaussian perturbations and
projection in both samplers. Their semantic sets then agree.
By \eqref{eq:sbog_threshold_lipschitz},
$|\widehat\tau_n-\tau_n^\star|\le t$.
Thus, for every shared proposal,
\[
\left|
(\widehat E_n(V_j)-\widehat\tau_n)
-(E(V_j)-\tau_n^\star)
\right|\le2t.
\]
Outside the threshold event in
\eqref{eq:sbog_proposal_margin}, the feasible sets agree.
If this common set contains at least two candidates and
$\Delta_n>2t$, its population maximizer $j^\star$ satisfies
\[
\widehat E_n(V_{j^\star})-\widehat E_n(V_j)
\ge E(V_{j^\star})-E(V_j)-2t>0
\qquad(j\ne j^\star).
\]
Hence the selected candidate agrees as well.
Agreement is immediate when the feasible set has at most one element.

Let $W_n$ and $W^\star$ be the outputs of this coupled round,
recording the accepted pair $(y_I,V_j)$ or a failure symbol $\bot$.
The preceding argument and a union bound give
\[
\Pr(W_n\ne W^\star)\le a_n(t)+b_n(t)=\eta_n.
\]
Consequently, for every measurable set $B$ of class--latent pairs,
\[
|\mu_n(B)-\mu_n^\star(B)|\le\eta_n,
\qquad
|p_n-p_n^\star|\le\eta_n,
\]
where $\mu_n(B)=\Pr(W_n\in B)$,
$\mu_n^\star(B)=\Pr(W^\star\in B)$,
and $p_n=\Pr(W_n\ne\bot)$.

\emph{Repeat-until-feasible sampling.}
Since $p_n\ge p_n^\star-\eta_n\ge p_0-\eta_n>0$,
independent proposal rounds terminate almost surely.
Summing over the number of failed rounds gives
\[
Q_n^{\mathrm{SBOG}}(B)=\frac{\mu_n(B)}{p_n},
\qquad
Q_{n,\mathrm{SOE}}^\star(B)=\frac{\mu_n^\star(B)}{p_n^\star}.
\]
Therefore,
\begin{align*}
\left|
Q_n^{\mathrm{SBOG}}(B)-Q_{n,\mathrm{SOE}}^\star(B)
\right|
&\le
\frac{|\mu_n(B)-\mu_n^\star(B)|}{p_n^\star}
+\frac{\mu_n(B)}{p_n}
 \frac{|p_n-p_n^\star|}{p_n^\star}\\
&\le\frac{2\eta_n}{p_n^\star}
\le\frac{2\eta_n}{p_0}.
\end{align*}
Taking the supremum over $B$ and using
$d_{\mathrm{TV}}\le1$ proves
\eqref{eq:sbog_population_tv}.
The convergence statement follows on the stated events.
For this statement, assign an arbitrary output law on exceptional
banks where the sampler is undefined; their probability vanishes
under the stated assumptions.

\emph{Decoding.}
Let $D(dx\mid c,v)$ be the shared conditional decoder, including
$p(dx\mid v)$ as a special case.
For every measurable input-space set $B$, the function
$D(B\mid c,v)$ lies in $[0,1]$. Integration of such a function
against the difference of two probability measures is bounded by
their total variation distance. Hence
\[
d_{\mathrm{TV}}\!\left(
Q_n^{\mathrm{SBOG}}D,Q_{n,\mathrm{SOE}}^\star D
\right)
\le d_{\mathrm{TV}}\!\left(
Q_n^{\mathrm{SBOG}},Q_{n,\mathrm{SOE}}^\star
\right).
\]
This proves the decoded-distribution claim.
\end{proof}

\paragraph{A sufficient uniform estimation condition.}
We now give one setting in which the required uniform energy error
vanishes. Fix the encoder before drawing an i.i.d. reference bank
$z_1,\ldots,z_n\sim\bar\nu$. Suppose
$\mathcal Z\subset\mathbb R^m$ is compact, $d\ge0$, and
\[
|d(v,u)-d(w,u)|\le L_d\|v-w\|_2
\quad\text{for all }u,v,w\in\mathcal Z.
\]
Assume $K_0=\inf_{v\in\mathcal Z}K_{\varepsilon}^{\bar\nu}(v)>0$.
Let $\mathcal N_{\mathcal Z}(h)$ be the size of a finite Euclidean
$h$-net of $\mathcal Z$, and define
\[
u_n(h,\delta)=
\sqrt{\frac{\log\{2\mathcal N_{\mathcal Z}(h)/\delta\}}{2n}}
+\frac{2L_dh}{\varepsilon}.
\]
If $u_n(h,\delta)\le K_0/2$, then, with probability at least
$1-\delta$,
\begin{equation}
\label{eq:sbog_uniform_energy_bound}
\sup_{v\in\mathcal Z}|\widehat E_n(v)-E(v)|
\le e_n(h,\delta)
:=\frac{2\varepsilon}{K_0}u_n(h,\delta).
\end{equation}

To prove this claim, note that
$u\mapsto\exp\{-d(v,u)/\varepsilon\}$ takes values in $[0,1]$.
Hoeffding's inequality and a union bound over an $h$-net control the
kernel-support error on that net by
$\sqrt{\log\{2\mathcal N_{\mathcal Z}(h)/\delta\}/(2n)}$.
The kernel is $L_d/\varepsilon$-Lipschitz in its first argument.
Moving from any query to a net point contributes at most
$2L_dh/\varepsilon$ to the difference between empirical and
population supports. Thus their uniform difference is at most
$u_n(h,\delta)$.
Both supports are then at least $K_0/2$, where
$x\mapsto-\varepsilon\log x$ is $2\varepsilon/K_0$-Lipschitz.
This proves \eqref{eq:sbog_uniform_energy_bound}.
Because the bound is uniform over $\mathcal Z$, it also applies to
the data-dependent anchor locations.

For the normalized latent space, one may use
$\mathcal N_{\mathcal Z}(h)\le(1+2/h)^m$.
Taking $h=n^{-1}$ and $\delta=n^{-2}$, with fixed
$m,\varepsilon,L_d,K_0$, gives
$e_n=O(\sqrt{\log n/n})$.
Consistency also requires the selection-margin terms to vanish.
For example, if, on events of probability tending to one,
\[
a_n(t)+b_n(t)\le C_{\mathrm{sel}}t^{\alpha}+v_n,
\qquad
\alpha>0,\quad v_n\to0,\quad p_n^\star\ge p_0>0,
\]
for all sufficiently small $t$, then the theorem gives
\[
d_{\mathrm{TV}}\!\left(
Q_n^{\mathrm{SBOG}},Q_{n,\mathrm{SOE}}^\star
\right)
\le\frac{2}{p_0}
\left(C_{\mathrm{sel}}e_n^{\alpha}+v_n\right)\to0
\]
on the corresponding events. These margin and acceptance conditions
are additional assumptions, not consequences of uniform energy
estimation alone.
The concentration calculation above concerns an i.i.d. global
reference bank; it does not by itself establish the same bound for
data-dependent trimming, local-reference restrictions, or an encoder
fitted on the same reference sample.

\paragraph{Boundary precision and coverage.}
The following consequence separates population-SOE consistency from
alignment with an independently specified target boundary.

\begin{corollary}[Conditional boundary precision and coverage]
\label{cor:sbog_boundary_coverage}
Fix a bank for which
$d_{\mathrm{TV}}(Q_n^{\mathrm{SBOG}},
Q_{n,\mathrm{SOE}}^\star)\le\zeta_n$.
For any measurable target boundary region $\mathcal B$,
\[
Q_n^{\mathrm{SBOG}}(\mathcal B)
\ge Q_{n,\mathrm{SOE}}^\star(\mathcal B)-\zeta_n.
\]
In particular, define $\operatorname{BPrec}(Q)=Q(\mathcal B)$.
If $Q_{n,\mathrm{SOE}}^\star(\mathcal B)=1$, then
$\operatorname{BPrec}(Q_n^{\mathrm{SBOG}})\ge1-\zeta_n$.

Let $\mathcal B_1,\ldots,\mathcal B_L$ be measurable target regions
such that
$\min_{\ell\le L}Q_{n,\mathrm{SOE}}^\star(\mathcal B_\ell)
\ge q_{\min}>\zeta_n$.
For $N$ independent outputs $U_1,\ldots,U_N$ from the fitted
empirical sampler, conditional on this bank,
\[
\Pr\!\left(
\text{every }\mathcal B_\ell\text{ contains at least one }U_j
\,\middle|\,\text{bank}\right)
\ge 1-L\exp\{-N(q_{\min}-\zeta_n)\}.
\]
The same statements hold for the decoded laws and input-space
regions when the common decoder is applied.
\end{corollary}

\begin{proof}
The precision and per-region mass bounds follow directly from the
definition of total variation. Conditional independence gives
\[
\Pr(U_j\notin\mathcal B_\ell\text{ for all }j\mid\text{bank})
=\{1-Q_n^{\mathrm{SBOG}}(\mathcal B_\ell)\}^N
\le\exp\{-N(q_{\min}-\zeta_n)\}.
\]
A union bound over the $L$ regions proves the coverage statement.
The decoded version follows from the decoder contraction established
in Theorem~\ref{thm:sbog_population_consistency}.
\end{proof}

The corollary is conditional on the fitted bank; outputs that share
this bank are not asserted to be independent unconditionally.
The coverage step uses only total variation and the assumed region
masses. It does not establish that the SOE-selected regions coincide
with a task-defined boundary.
Also, $Q_{n,\mathrm{SOE}}^\star$ depends on the observed bank.
Convergence to a fixed, sample-independent law
$Q_{\mathrm{SOE}}^\star$ would additionally require convergence of
these population-SOE counterparts. For example, an additional bound
$d_{\mathrm{TV}}(Q_{n,\mathrm{SOE}}^\star,
Q_{\mathrm{SOE}}^\star)\le\lambda_n\to0$ would give the bound
$\zeta_n+\lambda_n$ to that fixed law by the triangle inequality.

\section{Controlled Numerical Validation}
\label{sec:controlled_validation}

Downstream detection performance provides only indirect evidence of generation quality. We therefore first examine whether SBOG generates boundary outliers in a controlled setting where the target boundary is known and defined independently of SOE. The population oracle
uses three stable vector autoregressive (VAR) processes with perturbations $\Delta b_c$, $\Delta\Sigma_c$, and $\Delta A_c$ to their local means, innovation covariances, and transition dynamics, respectively. Using the corresponding exact population log-likelihood ratio $r_{c,k}(x)$, the target boundary band is $0<r_{c,k}(x)\le\delta_{c,k}$, where the indices identify the process and perturbation, and $\delta_{c,k}$ determines the selected band. This oracle is distinct from the population-SOE sampler in Theorem~\ref{thm:sbog_population_consistency}.

We compare SBOG with RandomNeg, kNNNeg, and KDENeg. We report boundary precision (BPrec), semantic validity (SemValid), average and worst-region boundary coverage (BCov and WorstCov), sliced Wasserstein-2 distance to the oracle (SW2), and detection AUPR on held-out oracle samples. Together, these measurements assess boundary alignment, coverage, and downstream utility.

\begin{table}[htbp]
\centering
\caption{Controlled boundary-generation evaluation against a
method-independent population oracle.}
\label{tab:population_oracle}
\footnotesize
\setlength{\tabcolsep}{2.8pt}
\renewcommand{\arraystretch}{1.12}
\begin{tabular}{@{}lcccccc@{}}
\toprule
Method
& BPrec$\uparrow$
& SemValid$\uparrow$
& BCov$\uparrow$
& WorstCov$\uparrow$
& \shortstack{SW2-to-\\oracle$\downarrow$}
& \shortstack{Oracle\\AUPR$\uparrow$} \\
\midrule
RandomNeg
& 0.229$\pm$0.027
& 0.950$\pm$0.001
& 0.091$\pm$0.036
& 0.011$\pm$0.004
& 0.2349$\pm$0.0034
& 0.455$\pm$0.005 \\
kNNNeg
& 0.298$\pm$0.036
& 0.977$\pm$0.014
& 0.033$\pm$0.005
& 0.002$\pm$0.002
& 0.1453$\pm$0.0062
& 0.684$\pm$0.014 \\
KDENeg
& 0.305$\pm$0.043
& 0.972$\pm$0.011
& 0.070$\pm$0.008
& 0.005$\pm$0.003
& 0.1489$\pm$0.0089
& 0.830$\pm$0.012 \\
SBOG
& \textbf{0.629}$\pm$\textbf{0.032}
& \textbf{0.983}$\pm$\textbf{0.013}
& \textbf{0.236}$\pm$\textbf{0.006}
& \textbf{0.012}$\pm$\textbf{0.003}
& \textbf{0.1405}$\pm$\textbf{0.0069}
& \textbf{0.896}$\pm$\textbf{0.052} \\
\bottomrule
\end{tabular}
\end{table}

Table~\ref{tab:population_oracle} shows that SBOG places more samples in the independently defined boundary band and achieves greater average boundary coverage than the competing generators, while maintaining high semantic validity. Its generated samples also support higher AUPR on held-out oracle data. Worst-region coverage nevertheless remains low, so these results do not establish uniform coverage of all boundary regions.
The experiment provides direct evidence of boundary alignment in this controlled setting, complementing rather than directly verifying the population-SOE consistency result.

\section{Time-Series Experimental Details}
\label{app:ts-details}

\subsection{Pipeline parameters}
\label{app:ts-params}

We keep the CARLA \cite{darban2025carla} pipeline fixed throughout
Section \ref{sec:ts_gen} and only swap the negative-sample construction module.
All five datasets (SMD, MSL, SMAP, SWaT, WADI) share the same encoder
architecture and optimization schedule; the only per-dataset quantity is the
input channel count $C$, which equals the number of variates in the
benchmark.

\paragraph{Pretext stage (representation learning).}
We use a 1-D ResNet encoder with mid-channel width $4$ and an MLP projection
head of output dimension $4$, so the latent space $\mathcal{Z}$ is a $4$-D
unit sphere after $\ell_2$ normalization.
Pretext training uses the contrastive objective with temperature
$\tau_{\mathrm{ctr}}=0.4$, batch size $50$, AdamW with $\mathrm{lr}=10^{-3}$ and weight decay $10^{-2}$ with cosine schedule for $30$ epochs, and the augmentation pipeline of
\cite{darban2025carla} with noise $\sigma_{\mathrm{aug}}=0.01$. The
pseudo-label space has $K=10$ classes obtained by the standard CARLA
neighborhood-clustering procedure.

\paragraph{Negative-sample construction (SOE sampler).} For all five time-series benchmarks, we use the same SOE sampler hyperparameters. We set the Sinkhorn regularization to $\varepsilon=0.02$ and define the energy threshold as $\tau_E=\rho_q+m$, where $\rho_q$ is the $95$-th percentile of the refined ID energies and the margin is $m=0.04$. For proposal generation, we use isotropic Gaussian noise with standard deviation $\sigma=0.02$ and draw $M=128$ proposals per anchor. The anchor pool contains the top-$A$ high-energy ID points, with $A=\lceil0.02n\rceil$ capped at $256$ and lower bounded by $8$ for small entities. When constructing the reference measure $\widehat{\nu}_n$, we trim the top $8\%$ highest-energy ID points to reduce boundary contamination.

\paragraph{Classification stage (downstream detector).}
We use the same ResNet backbone, $5$-NN soft-assignment, $50$ epochs of
Adam ($\mathrm{lr}=10^{-2}$, weight decay $10^{-3}$) with constant
schedule, batch size $50$, entropy weight $2.0$, and inconsistency weight
$0$. Anomaly contamination of the training pool is fixed at $0.99$.
Detection scores are reported point-wise following \cite{kim2022towards}.

\subsection{Quality and stability metrics}
\label{app:ts-metrics}

For every (sampler, entity) pair we generate negatives
$\{v_i\}_{i=1}^N \subset \mathcal{Z}$ and evaluate them against the
in-distribution embeddings $Z_n$ and pseudo-labels $\{y_i\}$ on the
\emph{same} latent space the encoder produces. All metrics are computed
on the $4$-D unit sphere with cosine ground cost. Let $z(v)$ denote the
1-NN of $v$ in $Z_n$ under $d_{\cos}$, $\mathrm{src}(v_i)$ the index of
the boundary anchor that proposed $v_i$, and $\mathcal{N}_k(v)$ its
$k$ nearest ID neighbors.

\paragraph{Energy gap  (Table 2).}
$\mathrm{Gap}=\widehat E(v) - \tau_E$, averaged over all generated
negatives. Measures by how much each accepted negative exceeds the
acceptance threshold. Larger gaps mean negatives sit deeper in the
high-energy outlier region, but as $\tau_E$ rises into the right tail of
the ID energy distribution the gap shrinks because $\tau_E$ rises faster
than $\widehat E(v)$ can.

\paragraph{Feasible fill ratio (Table 2).}
Fraction of ID anchors with at least one of $M$ proposals satisfying
$\widehat E(v_j)\ge\tau_E$,
\begin{equation}
\mathrm{Fill}=\frac{1}{n}\sum_{i=1}^n \mathbf{1}\!\left\{
\exists j\in[M]:\ \widehat E(v_j^{(i)})\ge \tau_E
\right\}.
\end{equation}

\paragraph{$\mathrm{MMD}^2$ (RBF, median heuristic, Figure 2).}
Squared maximum mean discrepancy between the negative and ID samples,
\begin{equation}
\mathrm{MMD}^2(\widehat P_v, \widehat\nu_n)
= \frac{1}{N^2}\sum_{i,j} k(v_i,v_j)
+ \frac{1}{n^2}\sum_{i,j} k(z_i,z_j)
- \frac{2}{Nn}\sum_{i,j} k(v_i,z_j),
\end{equation}
with Gaussian kernel $k(a,b)=\exp(-\gamma\|a-b\|_2^2)$ and bandwidth
$\gamma=1/(2\sigma_{\mathrm{med}}^2)$, where $\sigma_{\mathrm{med}}$ is
the median pairwise Euclidean distance on a random $1024$-point subset of
the union $\widehat P_v\cup \widehat\nu_n$. Both samples are subsampled
to $2048$ points before evaluation.
\emph{Higher is better}: large $\mathrm{MMD}^2$ means the negatives form
a distribution that is detectable as different from ID.

\paragraph{Sliced Wasserstein-2 distance (Figure 2).}
The 2-Wasserstein distance averaged over $1000$ random unit-norm 1-D
projections,
\begin{equation}
\widehat{\mathrm{SW}}_2(\widehat P_v, \widehat\nu_n)
= \sqrt{\frac{1}{L}\sum_{\ell=1}^L W_2^2\!\left(\theta_\ell^\top \widehat P_v,\theta_\ell^\top \widehat\nu_n\right)},
\quad L=1000,
\end{equation}
with $W_2$ on each line computed in $O(n\log n)$ via sorting.

\paragraph{Intra-negative diversity (Figure 2).}
Mean pairwise cosine distance among generated negatives,
\begin{equation}
\mathrm{Diversity}=\frac{1}{N(N-1)}\sum_{i\ne j}\big(1-v_i^\top v_j\big).
\end{equation}

\paragraph{Anchor fidelity (intended-class purity, Figure 2).}
Fraction of negatives whose $1$-NN in $Z_n$ shares the pseudo-label of
the boundary anchor that proposed them,
\begin{equation}
\mathrm{Fidelity}
= \frac{1}{N}\sum_{i=1}^N \mathbf{1}\!\left\{
y_{z(v_i)} \;=\; y_{\mathrm{src}(v_i)}
\right\}.
\end{equation}

\paragraph{Contamination (Table 7).} $\mathrm{Contam.} = 1 - \mathrm{Fidelity}$.

\paragraph{Cross-class margin (Table 7).}
Let $\mu_c$ be the (re-normalized) mean of $Z_n$ within pseudo-class $c$.
For each negative $v_i$ with source class $c^\star_i = y_{\mathrm{src}(v_i)}$,
\begin{equation*}
\mathrm{Margin}(v_i) = \cos(v_i,\mu_{c^\star_i}) - \max_{c\ne c^\star_i}\cos(v_i,\mu_c),
\qquad
\mathrm{Margin} = \frac{1}{N}\sum_i \mathrm{Margin}(v_i).
\end{equation*}

\begin{table}[h]
\centering\small
\caption{Component ablation of the TS outlier-synthesis sampler.}
\label{tab:ts-component-ablation}
\begin{tabular}{lccccc}
\toprule
Variant & Fidelity $\uparrow$ & Contam. $\downarrow$ & Margin $\uparrow$ & MMD$^2$ $\uparrow$ & Diversity $\uparrow$ \\
\midrule
Random & 0.339 $\pm$ 0.014 & 0.661 $\pm$ 0.014 & -0.000 $\pm$ 0.002 & 0.140 $\pm$ 0.043 & 0.051 $\pm$ 0.074 \\
w/o top-A anchor & 0.329 $\pm$ 0.025 & 0.671 $\pm$ 0.025 & 0.000 $\pm$ 0.002 & 0.467 $\pm$ 0.121 & 0.223 $\pm$ 0.062 \\
w/o $\rho$ threshold & 0.988 $\pm$ 0.004 & 0.012 $\pm$ 0.004 & 0.034 $\pm$ 0.038 & 1.076 $\pm$ 0.227 & 0.182 $\pm$ 0.026 \\
w/o argmax select & 0.989 $\pm$ 0.004 & 0.011 $\pm$ 0.004 & 0.034 $\pm$ 0.039 & 1.049 $\pm$ 0.247 & 0.099 $\pm$ 0.022 \\
$k$-NN & 0.990 $\pm$ 0.031 & 0.010 $\pm$ 0.031 & 0.025 $\pm$ 0.030 & 0.934 $\pm$ 0.301 & 0.195 $\pm$ 0.027 \\
KDE & 0.992 $\pm$ 0.028 & 0.008 $\pm$ 0.028 & 0.028 $\pm$ 0.032 & 0.978 $\pm$ 0.265 & 0.198 $\pm$ 0.024 \\
\textbf{SBOG (ours)} & 0.999 $\pm$ 0.004 & 0.001 $\pm$ 0.004 & 0.082 $\pm$ 0.038 & 1.156 $\pm$ 0.227 & 0.202 $\pm$ 0.026 \\
\bottomrule
\end{tabular}
\end{table}

\section{Image Experimental Details}
\label{app:image-details}
\subsection{Configuration}

\paragraph{In-distribution datasets and classifiers.}
We follow the experimental protocol of recent image outlier-synthesis
works~\citep{npos, du2023dream, gao2025good, liao2025bood} and the
benchmark setting of~\citep{huang2021mos}. The ID classifier is
ResNet-34~\citep{he2016deep}, trained on CIFAR-100 and ImageNet-100~%
\citep{deng2009imagenet} with the standard cross-entropy objective.
The latent representation $z(x;\theta)\in\mathbb{R}^{768}$ is the
row-normalised penultimate feature.

\paragraph{Class text anchors.}
Class anchors $\{t_c\}_{c=1}^{C}$ are the unit-normalised text-token
embeddings produced by the text encoder of Stable Diffusion v1.5 on each
class name.

\paragraph{ID feature bank.}
For each class $c$ we cache $1000$ ID features, giving an ID tensor
$Z_n\in\mathbb{R}^{C\times 1000\times 768}$. All energies, anchors and
candidates are computed in this normalised cosine space.

\paragraph{Parameters in outlier energy.}
\begin{itemize}\setlength{\itemsep}{2pt}
  \item Smoothing parameter $\varepsilon=0.05$.
  \item Threshold $\tau_E$ calibrated per class as the $0.95$-quantile of
        $\widehat{E}(z_i)$ over $Z_n$, plus a margin $m$:
        $\tau_E=\rho+m$ with $m=0$ as default.
  \item Gaussian perturbation scale $\sigma=0.012$.
  \item Each anchor proposes $M=256$ candidates per round, up to
        $\texttt{max\_rounds}=10$ rounds, until the per-class accepted set
        reaches $N_{\text{ood}}=10000$.
  \item Local-energy reference $=$ the $128$ nearest ID neighbours of each
        anchor; seed-cosine floor $0.7$.
\end{itemize}

\paragraph{Decoder.}
Accepted latents are decoded back to image space by Stable Diffusion v1.5
conditioned on the class token of the source anchor; one image per
accepted latent. The decoded set $\mathcal{D}_{\text{ood}}$ is used as
auxiliary outlier exposure during detector calibration.

\subsection{Definition of Diagnostic Metrics}

This subsection collects the precise definitions of every metric that
appears in the diagnostic tables and figures of Section~4.2 (Tables~4,~6,~8
and Figure~3). Throughout, $Z_n=\{z_i\}_{i=1}^{n}\subset\mathcal{Z}$ denotes
the ID feature bank for the class under consideration, $t_c\in\mathcal{Z}$
its class text anchor, $\widehat{E}(\cdot)=E^{\widehat{\nu}_n}_{\varepsilon}(\cdot)$
the empirical Sinkhorn outlier energy of Definition~3.1 with smoothing
$\varepsilon$, and $\tau_E=\rho+m$ the energy threshold of Algorithm~1.
Per-class quantities are reported as mean $\pm$ standard deviation across
the $100$ ImageNet-100 classes.

\paragraph{Empirical decision margin~$\Delta$ (Figure 4).}
For each candidate $v$ in the accepted candidate pool we compute
$|\widehat{E}(v)-\tau_E|$, and report
\begin{equation}
  \Delta \;=\; \mathrm{median}\bigl\{|\widehat{E}(v)-\tau_E|:v\in\mathcal V\bigr\}.    
\end{equation}
A larger $\Delta$ means accepted/rejected candidates sit further from the
threshold, so each individual decision is more confident. This is the
quantity controlled by Theorem~3.4: the sample complexity scales as
$\varepsilon^{2}/(\underline K^{2}\Delta_E^{2})$, so $\Delta$ acts as the
operationally observable proxy for $\Delta_E$.

\paragraph{Flip Rate@$256$ (Figure 4).}
For each class we draw a calibration subset $S$ of size $n=256$ uniformly
without replacement from $Z_n$, recompute the empirical energy
$\widehat{E}^{(S)}$ from $S$ alone, and re-derive a class-specific
threshold $\tau_E^{(S)}$ by the same $0.95$-quantile rule. The flip rate is
\begin{equation}
  \mathrm{FlipRate@256}
  \;=\; \mathbb{E}_{S}\!\left[\frac{1}{|\mathcal V|}
  \sum_{v\in\mathcal V}
  \mathbf{1}\!\left\{
    \mathbf{1}\{\widehat{E}^{(S)}(v)\!\geq\!\tau_E^{(S)}\}
    \;\neq\;
    \mathbf{1}\{\widehat{E}(v)\!\geq\!\tau_E\}
  \right\}\right],
\end{equation}
averaged over $30$ random calibration subsets. Smaller is better:
flip rate $\to 0$ means small calibration sets reach the same accept/reject
decisions as the full ID bank, the regime predicted by Theorem~3.4 when the
margin $\Delta_E$ is well-conditioned.

\paragraph{Token cosine (Table 4).}
Mean cosine similarity to the class text token,
\begin{equation}
  \overline{\cos}_{\text{tok}}
  \;=\; \frac{1}{N}\sum_{j=1}^{N}\langle v_j,t_c\rangle.
\end{equation}

\paragraph{Diversity (Table~4).}
One minus the mean off-diagonal cosine of the candidate set,
\begin{equation}
  \mathrm{Diversity}
  \;=\; 1 \;-\;
  \frac{1}{N(N-1)}\sum_{i\neq j}\langle v_i,v_j\rangle.
\end{equation}
A larger value means candidates spread more uniformly along the boundary
rather than collapsing to a few directions.

\paragraph{Yield (Table~4).}
The product of the on-class pass rate and the mean Sinkhorn outlier energy
of the passing subset:
\[
  \mathrm{Yield}
  \;=\; \mathrm{Pass}\cdot\overline{E}_{\mathrm{pass}},
\]
where, for each anchor, $32$ Gaussian perturbations of scale $\sigma=0.012$
are drawn,
$\mathrm{Pass}=\Pr[\cos(v,t_c)\!\geq\!\tau_{\text{sem}}]$ is the fraction
that survives the semantic filter, and $\overline{E}_{\mathrm{pass}}$ is
the mean Sinkhorn outlier energy of those passing perturbations. Yield
captures both axes simultaneously: a strategy must produce on-class
proposals \emph{and} those proposals must be high-energy.

\paragraph{Intended-class purity (Table~6).}
Fraction of candidates whose nearest ID neighbour over \emph{all} $C$
classes lies in the intended class $c$:
\begin{equation}\label{eq24}
  \mathrm{Purity}_c
  \;=\; \frac{1}{N}\sum_{j=1}^{N}
        \mathbf{1}\!\left\{\arg\max_{i\in\bigcup_{c'} Z_n^{c'}}
        \langle v_j,z_i\rangle\;\in\;Z_n^{c}\right\},
\end{equation}
where $Z_n^{c'}$ is the ID bank of class $c'$ (we use $150$ samples per
class for this nearest-neighbour bank). Higher $=$ the candidate stays
inside the correct class manifold.

\paragraph{Margin vs.\ other classes (Table~6).}
\begin{equation}\label{eq25}
  \mathrm{Margin}
  \;=\; \frac{1}{N}\sum_{j=1}^{N}
        \Bigl[\langle v_j,t_c\rangle
              - \max_{c'\neq c}\langle v_j,t_{c'}\rangle\Bigr].
\end{equation}

\paragraph{Contamination (Table~6).}
Fraction of candidates whose \emph{nearest class token} is not the
intended class:
\begin{equation}\label{eq26}
  \mathrm{Contamination}
  \;=\; \frac{1}{N}\sum_{j=1}^{N}
        \mathbf{1}\!\left\{\arg\max_{c'}\langle v_j,t_{c'}\rangle \neq c\right\}.
\end{equation}

\paragraph{Fidelity.}
Mean cosine similarity between each generated negative and its source
anchor's pseudo-class prototype $\mu_{y(v_j)}$, averaged over
$\mathcal{N}$. High fidelity means generated negatives are still
recognisably class-conditional rather than generic noise.

\paragraph{Contamination.}
Per-negative defined as
$\mathrm{Contam}(v_j)=\mathbf{1}\{\arg\max_{c'}\langle v_j,\mu_{c'}\rangle\neq y(v_j)\}$,
averaged over $\mathcal N$. Lower is better; together with Fidelity it
satisfies $\mathrm{Fidelity}+\mathrm{Contam}\approx 1$ in the
matched-prototype regime.

\paragraph{Margin.}
$\mathrm{Margin}(v_j)=\langle v_j,\mu_{y(v_j)}\rangle
                     -\max_{c'\neq y(v_j)}\langle v_j,\mu_{c'}\rangle$,
averaged over $\mathcal N$. Positive margin $=$ negatives sit closer to
their pseudo-class than to any other.

\paragraph{$\mathrm{MMD}^2$.}
Squared maximum-mean-discrepancy between the generated-negative
distribution $\widehat{P}_{\mathcal N}$ and the empirical positive
distribution $\widehat{P}_{Z_n}$,
\[
  \mathrm{MMD}^2(\widehat{P}_{\mathcal N},\widehat{P}_{Z_n})
  \;=\;
  \mathbb{E}\!\bigl[k(v,v')\bigr]
  + \mathbb{E}\!\bigl[k(z,z')\bigr]
  - 2\mathbb{E}\!\bigl[k(v,z)\bigr],
\]
with a Gaussian RBF kernel $k(\cdot,\cdot)$ at the median heuristic
bandwidth. Larger $=$ generated negatives induce a stronger
distributional shift away from positives, which is desirable for
contrastive negative mining.

\paragraph{Intra-negative diversity.}
$1-\frac{1}{N(N-1)}\sum_{i\neq j}\langle v_i,v_j\rangle$, the same
formula as ``Diversity'' in Table~4 above, applied to the time-series
generated-negative set. Higher $=$ negatives cover a wider range of
boundary directions instead of clustering on a single axis.

The last result shows the factors that influence the quality of the generation. Consider the two smoothing parameter $\epsilon$ in the outlier energy $E(v)=-\epsilon \log \int \exp\!\bigl(-d(v,u)/\epsilon\bigr)d\nu(u)$, and the energy margin $m$ in the acceptance rule $E(v)\ge \rho+m$. The four panels in Figure~\ref{fig:energy_ablation} respectively summarize their influence on energy separation (Energy Gap and Energy Strength), semantic preservation $\cos(v,t_c)$ (Semantic Alignment), and sampling feasibility (Feasible Fill Ratio). 

\begin{figure}[h]
    \centering
    \includegraphics[width=1.0\linewidth]{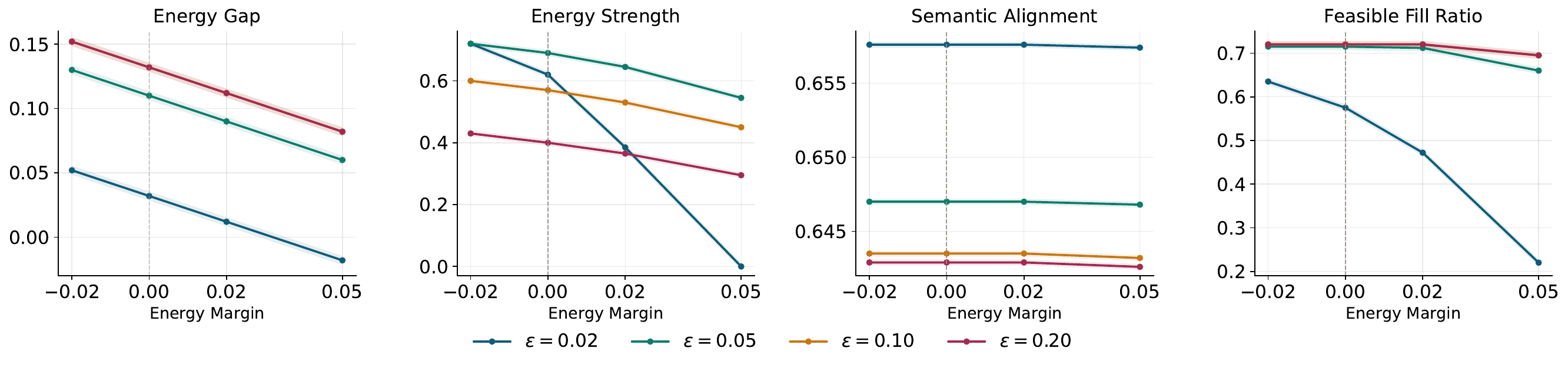}
    \caption{Generation quality validation}\vspace{-10pt}
    \label{fig:energy_ablation}
\end{figure}

As $m$ increases, the acceptance condition becomes stricter, which consistently reduces the energy advantage and the number of feasible synthesized outliers; this effect is especially severe for very small $\epsilon$, where the energy landscape becomes overly sharp and the feasible fill ratio drops rapidly. By contrast, semantic alignment remains nearly unchanged across settings, suggesting that the dominant trade-off is not semantic drift but the balance between outlier strength and synthesis stability. Overall, the results indicate that a moderate $\epsilon$ together with a small margin provides the most reliable setting, preserving strong outlierness while maintaining stable and feasible synthesis.

\end{document}